\documentclass[11pt]{article}

\usepackage[ruled, vlined, linesnumbered]{algorithm2e} 

\def\colorful{0}
\ifnum\colorful=1
\newcommand{\new}[1]{{\color{red} #1}}

\newcommand{\inote}[1]{\footnote{{\bf [Ilias: {#1}\bf ]}}}

\else
\newcommand{\new}[1]{#1}

\newcommand{\inote}[1]{}
\newcommand{\anote}[1]{}
\newcommand{\tnote}[1]{}
\newcommand{\snote}[1]{}
\newcommand{\lnote}[1]{}

\fi

\usepackage{mathrsfs}
\usepackage[cal=boondoxo]{mathalfa}
\def\eP{{\widehat{\mathcal{\sP}}}}

\def\ppcorrupt{{\sP_{\text{train}}}}
\def\pptest{{\sP_{\text{test}}}}

\usepackage{caption}

\newcommand{\thelabel}{y} 
\def\dual{\alpha}
\def\vdual{\valpha}
\newcommand{\extradual}{\beta} 
\newcommand{\vextradual}{\vbeta}

\usepackage{todonotes}
\usepackage{enumitem}

\usepackage[utf8]{inputenc} \usepackage[T1]{fontenc}    \usepackage{url}            \usepackage{booktabs}       \usepackage{amsfonts}       \usepackage{nicefrac}       \usepackage{microtype}      \usepackage{xcolor}         \usepackage[
style=alphabetic,
maxbibnames=99,
giveninits=true,
sortcites=true,
sorting=ynt,
backend=biber]{biblatex}
\def\citet{\cite}
\usepackage[normalem]{ulem}
\usepackage{tabularx}
\usepackage{float}
\usepackage{times}
\usepackage{amssymb, graphicx}

\usepackage{mathtools} \usepackage{amsthm}
\usepackage{thm-restate}

\newcounter{thm}
\numberwithin{thm}{section}

\theoremstyle{plain}
\newtheorem{theorem}[thm]{Theorem}
\newtheorem{corollary}[thm]{Corollary}
\newtheorem{lemma}[thm]{Lemma}
\newtheorem{proposition}[thm]{Proposition}

\newtheorem{assumption}[thm]{Assumption}
\newtheorem{fact}[thm]{Fact}

\theoremstyle{definition}
\newtheorem{definition}[thm]{Definition}

\theoremstyle{remark} \newtheorem{remark}[thm]{Remark}

\newtheorem{example}[thm]{Example}
\usepackage{hyperref}
\usepackage{cleveref}
\hypersetup{hypertexnames=false}

\crefname{fact}{fact}{facts}
\crefname{claim}{claim}{claims}
\crefname{assumption}{assumption}{assumptions}
\crefname{problem}{problem}{problems}
\crefname{theorem}{theorem}{theorems}
\crefname{corollary}{corollary}{corollaries}
\crefname{lemma}{lemma}{lemmas}
\crefname{proposition}{proposition}{propositions}
\crefname{definition}{definition}{definitions}
\crefname{remark}{remark}{remarks}
\crefname{note}{note}{notes}
\crefname{example}{example}{examples}
\crefalias{AlgoLine}{line}

\usepackage{times}

\usepackage{bm}

\usepackage{dsfont}

\usepackage{color}

\def\1{\bm{1}}

\def\eps{{\epsilon}}

\def\vzero{{\bm{0}}}
\def\vone{{\bm{1}}}
\def\vmu{{\bm{\mu}}}
\def\vxi{{\bm{\xi}}}

\def\valpha{{\bm{\alpha}}}
\def\vbeta{{\bm{\beta}}}

\def\ve{{\bm{e}}}

\def\vp{{\bm{p}}}

\def\vu{{\bm{u}}}
\def\vv{{\bm{v}}}
\def\vw{{\bm{w}}}
\def\vx{{\bm{x}}}

\def\vz{{\bm{z}}}

\def\vxh{\widehat{\bm{x}}}

\def\mA{{\bm{A}}}
\def\mB{{\bm{B}}}

\def\mI{{\bm{I}}}

\def\mM{{\bm{M}}}

\def\mX{{\bm{X}}}

\def\mSigma{{\bm{\Sigma}}}

\def\cB{{\mathcal{B}}}

\def\cE{{\mathcal{E}}}

\def\cG{{\mathcal{G}}}

\def\cI{{\mathcal{I}}}

\def\cO{{\mathcal{O}}}

\def\cS{{\mathcal{S}}}

\def\cU{{\mathcal{U}}}

\def\sD{{\mathbb{D}}}

\def\sP{{\mathbb{P}}}
\def\sQ{{\mathbb{Q}}}
\def\sR{{\mathbb{R}}}

\newcommand{\E}{\mathbb{E}}

\newcommand{\R}{\mathbb{R}}

\newcommand{\Ind}{\mathbb{I}}

\DeclareMathOperator*{\argmax}{arg\,max}
\DeclareMathOperator*{\argmin}{arg\,min}
\DeclareMathOperator{\TV}{TV}

\newcommand{\Gap}{\mathrm{Gap}}

\newcommand{\dd}{\mathrm{d}}

\DeclareMathOperator{\Var}{Var}

\def\Cov{{\bm{\mathrm{Cov}}}}

\DeclareMathOperator{\op}{op}
\usepackage{xparse}
\NewDocumentCommand{\DeclarePairedDelimiterWithOptionalStar}{ m m m m }{
  \DeclareDocumentCommand{#1}{ s o m }{\IfBooleanTF{##1}{#2#3*{##3}#4}{\IfValueTF{##2}{#2#3[##2]{##3}#4}{#2#3{##3}#4}}}}

\DeclarePairedDelimiter\brackets{(}{)}
\DeclarePairedDelimiter\norm{\|}{\|}
\DeclarePairedDelimiter\abs{|}{|}
\DeclarePairedDelimiterX{\divx}[2]{(}{)}{#1\;\delimsize\|\;#2}

\DeclarePairedDelimiterWithOptionalStar{\bigO}{O}{\brackets}{}
\DeclarePairedDelimiterWithOptionalStar{\bigtO}{\widetilde{O}}{\brackets}{}
\DeclarePairedDelimiterWithOptionalStar{\bigOmega}{\Omega}{\brackets}{}
\DeclarePairedDelimiterWithOptionalStar{\bigtOmega}{\widetilde{\Omega}}{\brackets}{}
\DeclarePairedDelimiterWithOptionalStar{\opnorm}{}{\norm}{_{\op}}
\DeclarePairedDelimiterWithOptionalStar{\fnorm}{}{\norm}{_{F}}

\makeatletter
\providecommand*{\diff}{
  \@ifnextchar^{\DIfF}{\DIfF^{}}}

\def\DIfF^#1{\mathop{\mathrm{\mathstrut d}}\nolimits^{#1}\gobblespace}
\def\gobblespace{\futurelet\diffarg\opspace}
\def\opspace{\let\DiffSpace\!\ifx\diffarg(\let\DiffSpace\relax
  \else
  \ifx\diffarg[\let\DiffSpace\relax
  \else
  \ifx\diffarg\{\let\DiffSpace\relax
  \fi\fi\fi\DiffSpace}

\newcommand{\numhere}{\stepcounter{equation}\tag{\theequation}}

\title{Distributionally Robust Optimization\\ 
with Adversarial Data Contamination}

\author{Shuyao Li\thanks{Supported in part by AFOSR Awards FA9550-21-1-0084 and FA9550-24-1-0076, the U.S.\ Office of Naval Research under award number N00014-22-1-2348, and by NSF CAREER Award CCF-2440563.}\\
  University of Wisconsin-Madison\\
  \texttt{shuyao.li@wisc.edu}\\ 
  \and
  Ilias Diakonikolas\thanks{Supported by NSF Medium Award CCF-2107079 and an H.I. Romnes Faculty Fellowship.}\\
  University of Wisconsin-Madison \\
  \texttt{ilias@cs.wisc.edu} \\ 
  \and
  Jelena Diakonikolas\thanks{Supported in part by the AFOSR YIP Award FA9550-24-1-0076, by the U.S.\ Office of Naval Research under contract number N00014-22-1-2348, and by the NSF CAREER Award CCF-2440563.}\\
  University of Wisconsin-Madison \\
  \texttt{jdiakonikola@wisc.edu}
}\date{}

\DeclareMathOperator{\CVaR}{CVaR}

\DeclareMathOperator{\Proj}{Proj}

\DeclareMathOperator{\LP}{LP}
\def\mathbf{\bm}

\begin{document}
\maketitle

\setcounter{page}{0}

\thispagestyle{empty}

\begin{abstract}Distributionally Robust Optimization (DRO) provides a framework for decision-making under distributional uncertainty, yet its effectiveness can be compromised by outliers in the training data. This paper introduces a principled approach to simultaneously address both challenges. We focus on optimizing Wasserstein-1 DRO objectives for generalized linear models with convex Lipschitz loss functions, where an $\epsilon$-fraction of the training data is adversarially corrupted. Our primary contribution lies in a novel modeling framework that integrates robustness against training data contamination with robustness against distributional shifts, alongside an efficient algorithm inspired by robust statistics to solve the resulting optimization problem. We prove that our method achieves an estimation error of $O(\sqrt{\epsilon})$ for the true DRO objective value using only the contaminated data under the bounded covariance assumption. This work establishes the first rigorous guarantees, supported by efficient computation, for learning under the dual challenges of data contamination and distributional shifts. 
\end{abstract}

\newpage

\section{Introduction}

Distributionally Robust Optimization (DRO) is a general stochastic 
optimization framework, where the goal is to minimize the loss of a target model 
on a worst-case distribution that is promised to lie 
in an ambiguity set---a set of distributions close to the reference distribution, 
with respect to a prespecified divergence. Intuitively, the ambiguity set 
models possible distributional shifts of the data.
Since its introduction, 
the DRO framework has found a wide range of applications in a variety of 
contexts, including algorithmic fairness~\cite{hashimoto2018fairness}   
class imbalance~\cite{xu2020class}, 
reinforcement learning~\cite{Lotidis2023, Yang2023, Wang2023a, Yu2023, Kallus2022, Liu2022, ramesh2024distributionally}, robotics~\cite{Sharma2020}, 
sparse neural network training~\cite{Sapkota2023}, 
defense against model extraction~\cite{Wang2023b}, and language model pretraining~\cite{Liu2021, xie2023doremi}, finetuning~\cite{ma2025task}, pruning~\cite{xia2023sheared,deng2024drpruning}, and alignment~\cite{wu2025towards}.

Despite its advantages, 
the practical adoption of DRO is often hindered by the following vulnerability: 
its inherent sensitivity to outliers present in the training 
data; see, e.g.,~\citet{hu2018does, zhu2022generalized}.
This sensitivity is not merely a theoretical worst-case concern.  
Contaminated training data can significantly degrade 
the performance of DRO solutions, making it a critical barrier 
to its wider application in real-world scenarios 
where data quality is not always guaranteed~\cite{huang2023robust}. 
Notably, some DRO formulations, particularly those related to $f$-divergence-based 
measures (such as Conditional Value-at-Risk (CVaR)), 
provably assign higher importance to extreme data 
points~\cite{blanchet2024distributionally}.

The well-established field 
of robust statistics~\cite{Huber09, diakonikolas2022algorithmic} 
offers a natural and principled approach to modeling 
and mitigating the effects of data contamination 
in the training data. However, integrating these principles 
effectively within the DRO framework---in a manner that is both 
theoretically sound and computationally tractable---has turned out 
to be a considerable challenge. Many existing approaches that attempt to 
combine outlier robustness with distributional robustness struggle to 
maintain desirable properties or offer provable computational guarantees.

A key desirable property of any outlier-robust DRO formulation 
is that it behaves sensibly under limiting conditions. 
In particular, it should satisfy two sanity checks: 
(1) when the ambiguity set, describing distributional shifts, 
reduces to a single distribution 
(i.e., the training and test distributions are identical and known, 
aside from potential contamination in the training data), 
the problem should naturally reduce 
to a standard outlier-robust supervised learning task; 
and (2) when there are no outliers in the training data, 
the formulation should recover the standard DRO problem 
for supervised learning. Many existing models fail 
to meet one or both of these criteria; 
see \Cref{sec:dro-modeling} for details

In this work, we initiate a systematic investigation of DRO in the 
presence of adversarial data contamination within the training set. 
As mentioned in the preceding discussion, in sharp contrast to prior 
formulations proposed with a similar motivation, our formulation captures 
both the DRO and the outlier-robust settings as special cases. We view 
this as an important conceptual contribution of this work.
At the technical level, we aim to design the first statistically and 
computationally efficient 
algorithms for solving natural learning tasks robustly 
to distributional shifts, where an $\epsilon$-fraction 
of the training data can be adversarially corrupted 
according to the strong contamination 
model~\cite{diakonikolas2016robust-estimato}.
\begin{definition}[Strong Contamination Model]
	\label{def:corruption}
	Given a parameter \(0 < \eps < 1/2\) and an inlier distribution \(\sP_0\), an algorithm receives \emph{\textbf{\(\epsilon\)-corrupted}} samples of size $N$ from \(\sP_0\) as follows: $N$ clean samples 
    are first drawn i.i.d.\ from  \(\sP_0\). An adversary is then allowed to inspect these samples and replace an \(\epsilon\)-fraction of the samples with arbitrary points.
\end{definition}
The strong contamination model is standard in robust statistics and 
subsumes other models like total variation contamination, 
where the contaminated samples 
are drawn i.i.d.\ from an unknown distribution 
$\ppcorrupt$ satisfying\footnote{The total variation distance between two distributions $\sP$ and $\sQ$ is defined as \(\TV(\sP, \sQ) = \sup_A \abs{\sP(A) - \sQ(A)}\).} $\TV(\sP_0, \ppcorrupt) \le \eps$. Note here that, consistent with the literature in robust statistics, $\eps$ denotes the fraction of the outliers---not the target accuracy.

The technical problem we address is the computation of solutions 
to the Wasserstein-$1$ DRO problem, given access to 
$\epsilon$-corrupted training data. Specifically, 
given a reference distribution $\sP_0$, a radius $\rho > 0$, 
a loss function $\ell(\vw; \vxi)$, 
and a cost function $c: \R^{d+1} \times \R^{d+1} \to \R_+$ 
defining the Wasserstein-$1$ distance
$W_1^c(\sP, \sP') = (\inf_{\pi \sim \Pi(\sP, \sP')} \E_{(\vxi, \vxi') \sim \pi} [c(\vxi, \vxi')])$\footnote{Wasserstein-$p$ distance would be defined by $W_p^c(\sP, \sP') = (\inf_{\pi \sim \Pi(\sP, \sP')} \E_{(\vxi, \vxi') \sim \pi} [c^p(\vxi, \vxi')])^{1/p}$.},
our goal is to solve the problem\footnote{Here and elsewhere, ``min'' is interpreted as ``minimize,'' 
while ``inf'' would refer to its value---the infimum.}:
\begin{equation}\label{eq:w1-dro}
  \min_{\vw \in \R^d} \max_{\sQ: W_1^c(\sP_0, \sQ) \le \rho} \E_{\vxi \sim \sQ}[\ell(\vw; \vxi)],
\end{equation}
where we only have access to a dataset generated according to \Cref{def:corruption}. 
That is, we ask:
\begin{center}
  \em Is there an efficient algorithm for the DRO problem \eqref{eq:w1-dro} when the training data are $\epsilon$-corrupted? 
\end{center}
Our main result provides an affirmative answer for a range of supervised learning tasks, 
under the assumption that the underlying loss function satisfies certain well-defined (mild) assumptions. 

\begin{theorem}[Main Theorem--Informal; see \Cref{cor:population-error}]
Suppose the loss function $\ell(\vw; \vx, y)$ takes the form of $\phi(\vw \cdot \vx, y)$ 
for some convex function $\phi$ that is $\zeta$-Lipschitz in the first coordinate for all $y$. 
There is a sample size $N = \tilde O( d /\epsilon)$ and an algorithm that, given an $\eps$-corrupted set of samples 
of size $N$, with high probability solves 
problem \eqref{eq:w1-dro} approximately 
within $O(\norm{\vw^*}_2 \sigma \zeta \sqrt\epsilon)$ error, 
where $\vw^*$ is the optimal solution and $\sigma$ is the variance of the covariates $\vx$.
\end{theorem}

\new{The error guarantee in the above theorem scales naturally with the parameters of the problem, and we establish in \Cref{thm:lower_bound} that those dependencies are optimal.
While the dependence on the initial distance to the optimum is necessary, we also show that it can be removed for specific problems like SVMs under stronger anti-concentration assumptions, as detailed in \Cref{lemma:svm_clipping_w}.}

\subsection{Technical Overview}
Our work makes two primary contributions. 
At the conceptual level, we provide a novel formulation 
of DRO in the presence of outliers in the training data that we argue
is appropriate in a number of settings. 
At the technical level, we design the first computationally-efficient algorithm 
with provable error guarantees for our outlier-robust DRO formulation. 
\subsubsection{Modeling DRO with Adversarial Training Data Contamination} 
We start by pointing out that effectively formulating DRO 
in the presence of adversarial training data contamination is subtle. 
Many prior approaches attempt to handle outliers in the training set 
by encoding information about potential contamination directly 
within the ambiguity set~\cite{nietert2023outlier}, 
sometimes alongside modifications to the loss function~\cite{wang2024outlier}. 
While such models can, in principle, provide guarantees if the true 
data-generating distribution $\sP_0$ remains within this augmented ambiguity 
set centered around the contaminated empirical distribution, 
they often suffer from various drawbacks. 
As discussed in \Cref{sec:related-models}, these prior modeling strategies 
can lead to guarantees that depend undesirably on data dimensionality 
or may fail to recover standard robust estimation tasks 
in the limit of no distributional shift (i.e., when $\rho \to 0$). 
Moreover, as we argue in \Cref{sec:modeling-principles}, 
such formulations often implicitly treat training data outliers 
in a way that is more aligned with handling outliers (or model misspecification) 
in the \emph{testing} distribution.

In sharp contrast, our modeling framework treats pre-decision (training data contamination) 
and post-decision (distributional shift at test time) uncertainties separately, aligned 
with \citet{blanchet2024distributionally}. This separation allows for a 
more effective way to achieve robustness to both challenges simultaneously. 
We focus on the setting where an $\epsilon$-fraction of the training data 
is adversarially corrupted, and the goal is to find a model 
that performs well on a test distribution 
that may be shifted from the clean underlying distribution $\sP_0$.

\subsubsection{Formulating and Solving DRO with Adversarial Training Data Contamination}
Having established our modeling principles, 
the challenge shifts to formulating and then efficiently solving the resulting optimization problem. 
The interplay between DRO and outlier-robustness presents several subtleties 
that require careful consideration. Employing $f$-divergence-based ambiguity sets, 
as explored in  prior work~\cite{zhai2021doro,bennouna2022holistic} 
combining DRO with outlier awareness, 
could be counterproductive in modeling DRO with training data contamination. 
Many $f$-divergences, and their corresponding DRO formulations 
(like those related to CVaR), are inherently sensitive to extreme data points, 
potentially amplifying the negative effect of outliers rather than mitigating it~\cite{blanchet2024distributionally}.

As our main technical contribution, 
we identify an expressive and tractable setting 
by focusing on Wasserstein-1 DRO ($W_1$-DRO) with convex Lipschitz loss functions 
of the form $\ell(\vw; \vx, y) = \phi(\vw \cdot \vx, y)$, as outlined in \Cref{assump:convex-lipschitz}. We note that the Lipschitz assumption of 
$\ell_{y_i}(\cdot)$ is a mild condition, 
because any loss function with superlinear growth 
admits $+\infty$ worst-case loss under the setup of $W_1$-DRO, 
i.e., if $\ell_y(z) / \abs{z}^{1 + \iota} \to C$ as $\abs{z} \to \infty$ 
for some $\iota > 0$ and constant $C > 0$, 
then \(\sup_{\sQ: W_1(\sP, \sQ) \le \rho} \E_{\sQ}[\ell_y(\vw \cdot \vx)] = + \infty\) for all $\vw$ and $\rho > 0$ (see \Cref{fact:wasserstein-loss-moment}).

\begin{assumption}\label{assump:convex-lipschitz}
  The loss function is of the form $\ell(\vw; \vx, y) = \phi(\vw \cdot \vx, y)$ 
  for some univariate function $\phi$. For any label $y$, 
  we define \(\ell_y \equiv \phi(\cdot, y)\). We assume that (i) the function $\ell_{y}$ 
  is convex and $\zeta$-Lipschitz (but not necessarily smooth) for all $\thelabel$; 
  and (ii) the convex conjugate $\ell_{y}^*$ is proximable.
\end{assumption}
This specific combination of a loss function and the ambiguity set is important for our 
algorithmic results. Prior work from the Wasserstein DRO 
literature~\cite{shafieezadeh2019regularization} allows us to reformulate 
the $W_1$-DRO problem \eqref{eq:w1-dro} as an equivalent regularized 
risk minimization problem \eqref{eq:w1-dro-regularization}, 
which takes the empirical form \eqref{eq:optimization_formulation}  
\begin{equation}\label{eq:optimization_formulation}
  \min_{\vw \in \R^{d}} \frac1N \sum_{i=1}^{N} \ell_{y_i}(\vx_i \cdot \vw) + \psi(\vw)  \tag{P}
\end{equation}
This step is crucial, as it converts the original min-max DRO problem 
into a minimization-only problem. However, the adversarial contamination 
in the training data means that we cannot directly compute 
empirical expectations or gradients using clean samples from $\sP_0$. 
Instead, leveraging algorithmic results from robust statistics (\Cref{sec:preliminaries}), 
we show that the task effectively reduces to minimization 
of a generalized linear model with norm regularization, 
where access to the covariates is through an inexact oracle. 
This oracle processes the contaminated data and returns 
an estimate with bounded error $\Delta$, 
as formalized in \Cref{assump:approximation-guarantee} 
and demonstrated in \Cref{prop:inexact-hg-oracle}.
\begin{assumption}\label{assump:approximation-guarantee}
  There exists $\Delta > 0$ and an oracle that, 
  on input an arbitrary $3\zeta$-uniformly bounded sequence 
  $\vextradual = (\extradual^i)_{i=1}^N$ and an $\eps$-corrupted version 
  of a stable set (\Cref{fact:stability_bounded_covariance}) $\{\vx^s\}_{s=1}^N$, outputs $\hat{\vz}$ so that $\norm{\hat\vz - \frac1N \sum_{i=1}^{N} \extradual^i \vx^s_i}_2 \le \Delta$.
\end{assumption}
To efficiently solve this composite minimization problem with an inexact data oracle, 
we employ a primal-dual approach. By using Fenchel conjugacy 
(see \Cref{def:convex_conjugate}), we reformulate the objective 
to separate the non-linear but data-independent part of the loss function 
from the part that is linear in the data. 
This structural separation is vital, as it allows the robust mean estimation oracle 
to operate directly on appropriately weighted covariates 
to estimate the required linear term.

The core of our algorithmic approach is a variant 
of the Primal-Dual Hybrid Gradient (PDHG) method \cite{chambolle2011first}. 
This algorithm is adapted to our setting 
with two main distinctions from standard PDHG:
\begin{enumerate}[label=(\roman*), topsep=0pt, itemsep=-2pt, leftmargin=*]
    \item The primal update step (Line~\ref{line:primal-update} in \Cref{alg:main-opt}) 
    relies on the inexact oracle from \Cref{prop:inexact-hg-oracle} 
    to compute an estimate of the weighted sum of notionally clean covariates 
    from the available contaminated samples.
    \item The proximal term in the primal update includes a regularization parameter 
    $c_k$ that strategically \emph{decreases} over iterations, 
    thus increasing primal steps as the algorithm progresses. 
    This enables us to establish an estimation error guarantee (\Cref{prop:main-opt}) 
    that depends on the initial distance to the optimum, $\|\vw_0 - \vw^*\|_2$, 
    rather than on the potentially much larger diameter 
    of the region containing all iterates.
\end{enumerate}
A noteworthy feature of our algorithm and its analysis is 
that the robust mean estimation oracle is invoked only for the primal variable updates. 
In contrast, the dual variable updates (Line~\ref{line:dual-update} 
in \Cref{alg:main-opt}) are performed using the raw, 
potentially contaminated, data points. Our convergence analysis carefully 
handles this asymmetry and the inexactness stemming 
from the oracle, by demonstrating that the algorithm's iterates 
generated using contaminated data effectively track 
those of an idealized algorithm (\Cref{alg:main-opt-analysis}) 
that operates on underlying notionally clean (i.e., ``stable'') 
data with a similar bounded-error oracle. 

\section{Preliminaries\label{sec:preliminaries}}
Given a positive integer \(N\), define \([N] := \{1, 2, \dots, N\}\).
We use \(\Ind_{\cE}\) to denote the 0-1 indicator function of a set $\cE$. We use $\chi$ to denote the characteristic function  $\chi(0) = 0$ and $\chi(a) = +\infty$ otherwise. 
For vectors \(\vx\) and \(\vxh\) from the \(d\)-dimensional Euclidean space \(\mathbb{R}^{d}\), we use \(\langle \vx, \vxh \rangle\) and \(\vx \cdot \vxh\) to denote the standard inner product.  
We use \([x^{1}, x^{2}, \dots, x^{d}]^{\top}\) to denote the entries of \(\vx \in \mathbb{R}^d\).
For \(r > 1\), the \(\ell_r\) norm of a vector \(\vx \in \mathbb{R}^d\) is defined as \(\|\vx\|_r = ( \sum_{i=1}^d |x^i|^r )^{1/r}\). The \(\ell_\infty\) norm is defined as \(\|\vx\|_\infty = \max_{1 \leq i \leq d} |x^i|\).
The operator norm of a matrix \(\mA \in \mathbb{R}^{m \times n}\), induced by the \(\ell_2\) norm, is defined as \(\opnorm{\mA}  = \sup_{\|\vx\|_2 = 1} \|\mA \vx\|_2\). We write \( \mA \succeq \mB \) to indicate that  \(\vx^\top(\mA-\mB)\vx \geq 0\) for all $\vx \in \R^d$. 
For a set $S$ containing vectors of the same dimension, we let $\vmu_S$ denote their empirical mean and let $\mSigma_S$ denote their empirical covariance matrix.
For two functions \(f\) and \(g\), we say \(f = \tilde{O}(g)\) if \(f = O(g \log^k(g))\) for some constant \(k\), and similarly define \(\tilde{\Omega}\). 
\subsection{Optimization}
We collect here useful definitions and facts from optimization.
\begin{definition}[Proximable function]\label{def:proximable}
  A convex function $f: \R^d \to \R$ is \emph{proximable} if, for every $\vw \in \R^d$, the minimizer $\argmin_{\vw' \in \R^d} \big(f(\vw') + \frac{1}{2} \norm{\vw - \vw'}_2^2\big)$ can be computed efficiently.
  \end{definition}
\begin{definition}[Convex Conjugate]
  \label{def:convex_conjugate}
  Let \( l: \sR \to \sR \) be a convex function. The \emph{convex conjugate} (or Fenchel conjugate) of \( g \) is the function \( l^*: \sR \to \sR \cup \{+\infty\} \) defined by
  $ 
      l^*(\dual) \coloneqq \sup_{\hat\thelabel \in \sR} \{ \dual \hat\thelabel - l(\hat\thelabel) \}
$.
  \end{definition}

  \begin{definition}[Lipschitz Modulus]
    The Lipschitz modulus \(\zeta > 0\) of a function \(l: \R \to \R\) is $\inf\{\zeta': {\abs{l(\hat y_1) - l(\hat y_2)} \le \zeta' \abs{y_1 - y_2}}  \text{ for all \(\hat y_1, \hat y_2 \in \R\)}\}$.
    If $\zeta < \infty,$ we also say that $l$ is $\zeta$-Lipschitz.
  \end{definition}

\begin{fact}\label{fact:lipschitz-conjugate}
  Let $l: \R \to \R$ be convex. Given $L > 0$, $L$-Lipschitzness of $l$ is equivalent to $L$-boundedness of the domain of $l^*$:   
  $\forall y_1, y_2 \in \R, \, \abs{l(y_1) - l(y_2)} \le L \abs{y_1 - y_2} \Leftrightarrow (\forall \dual \in \sR,\, l^*(\dual) < \infty \Rightarrow \abs{\dual} \le L)$.
\end{fact}
\subsection{Robust mean estimation}
Recent developments in robust mean estimation algorithms achieve dimension-independent error guarantees under the strong contamination model.
A family of such algorithms rely on the notion of $(\epsilon, \delta)$-sample stability~\cite[Definition 2.1]{diakonikolas2022algorithmic}.
\begin{definition}[Stability (informal)]
\label{def:stability}
  A set \(S\) is \((\eps, \delta)\)-stable if removing any 
  $\eps$-fraction of the points will not change the mean by more than $\delta$ nor the variance in any direction by more than $\delta^2 / \eps$.
\end{definition}
The following characterization of the stability condition is convenient when (uncorrupted) data generating distribution $\sP_0$ is only assumed to have bounded operator norm of its covariance matrix:  
\begin{fact}[Stability under bounded covariance~{\cite[Lemma 3.11]{diakonikolas2022algorithmic}}]
  \label{fact:stability_bounded_covariance}
  A set $S$ is $(\epsilon, \delta)$-stable with respect to $\vmu$ and $\sigma^2$ where $\delta = O(\sqrt\epsilon)$ if and only if the following conditions hold: (i) \(\norm{\vmu_S - \vmu}_2 \le \epsilon\) and (ii) \(\opnorm{\Cov(S)} \le O(\sigma^2)\).
\end{fact}
One such algorithm is provided in \Cref{sec:appendix-prelim}. State-of-the-art algorithms for robust mean estimation can be implemented in near-linear time, requiring only a logarithmic number of passes over the data; see, e.g.,~\citet{diakonikolas22streaming, ChengDG19, dong2019quantum}.
Any such faster algorithm could be adopted for our purposes.
\begin{fact}[{Robust Mean Estimation with Stability~\cite[Theorem A.3]{diakonikolas2020outlier}}]\label{fact:robust_mean_estimation_with_stability}
  Let $T \subset \mathbb{R}^k$ be an $\epsilon$-corrupted version of a set $S$, where $S$ is $(O(\epsilon), \delta)$-stable with respect to $\vmu_S$ and $\sigma^2$. Algorithm~\ref{alg:robust_mean_estimation} on input $\epsilon$ and  $T$ (but not $\sigma$ or $\delta$) deterministically returns a vector $\hat{\vmu}$ in polynomial time so that $\norm{\vmu_S-\hat{\vmu}}_2=O(\sigma \delta)$.
\end{fact}
\Cref{fact:robust_mean_estimation_with_stability} requires uncorrupted samples to be stable. With a sufficiently large sample size, most samples from a bounded covariance distribution are stable. The remaining samples can be treated as outliers without an effect on the (order of the) estimation error.
\begin{fact}[{Sample Complexity for Stability~\cite[Proposition 3.9]{diakonikolas2022algorithmic}
}
  ]\label{fact:bounded_variance_stability_iid}
  Let $S$ be a set of $N$ independent samples from a distribution on $\mathbb{R}^d$ with mean $\vmu$ and covariance $\mSigma$. Fix $\tau \in (0, 1)$. With probability at least $1-\tau$, the subset $S' = \{\vx \in S: \norm{\vx - \vmu_S}_2 \le2 \sqrt{\opnorm{\mSigma}} \sqrt{d / \epsilon}\}$ satisfies $\left|S^{\prime}\right| \geq\left(1-\epsilon\right) |S|$ and $S^{\prime}$ is $\left(\epsilon, O(\sqrt\epsilon)\right)$-stable with respect to $\vmu$ and $\opnorm{\mSigma}$, provided the  sample size \(N = \bigO[\big]{(d \log d + \log (1 / \tau))/{\epsilon}}\) is sufficiently large.
\end{fact}

\section{Data-contaminated DRO: Modeling Principles and Problem Formulation\label{sec:dro-modeling}}

In this section, we develop our conceptual framework for addressing data contamination and distributional shifts simultaneously. We begin by establishing our main modeling principles, followed by a critical review of the most closely related models to motivate our specific formulation. We then argue that under standard assumptions about the loss function and the reference (generating) distribution, we can reduce the considered problem to a convex optimization problem with an inexact data oracle.  

\subsection{Modeling Principles\label{sec:modeling-principles}}

The modeling principles follow similar considerations as  \citet{blanchet2024distributionally}. Let \(\sP_0\) denote the data generating distribution (also called the reference distribution): this is the distribution from which clean, uncorrupted training samples would be drawn if there were no data corruptions. The training samples passed to a learning algorithm might contain outliers, be contaminated, or otherwise have discrepancy with \(\sP_0\); thus we let \(\ppcorrupt\) denote the contaminated distribution that generates the training data. A sample set of size \(N\) is then drawn from \(\ppcorrupt\), leading to the empirical distribution \(\eP = \eP(N)\). The learner has access to \(\eP(N)\), based on which it outputs a model parameterized by a vector \(\vw\). The model parameterized by \(\vw\) is tested against an out-of-sample distribution (also known as testing distribution) \(\pptest\), which may deviate from $\sP_0$. This process can be summarized as: 
\begin{equation}\label{eq:decision}
  \sP_0 \xrightarrow[]{\text{contamination}} \ppcorrupt \xrightarrow[]{\text{sampling}} \eP(N) \xrightarrow[]{\text{train-to-test shift}} \pptest.
\end{equation}

The model parameterized by $\vw$ might perform poorly on $\pptest$ for various reasons:
(i) Overfitting/statistical error: $\eP(N) \neq \pptest$; (ii) Distributional shift: $\pptest \neq \sP_0 $; or (iii) Data contamination: $\sP_0 \neq \ppcorrupt$. 
Errors due to (i) and (ii) arise in the post-decision stage, while the error in (iii) is incurred in the pre-decision stage. 
Distributionally robust optimization (DRO) minimizes the loss with respect to distributions in some uncertainty (or ``ambiguity'') set that is specified without inspecting the training data; as a consequence, it is best used to model post-decision errors. 
In this work, we specifically focus on DRO's advantage in guarding against distributional shifts. We gloss over the overfitting errors by drawing samples with a sufficiently large sample set size $N$ and apply uniform convergence analysis to guarantee the performance of every model $\vw$ on $\ppcorrupt$. 

\subsection{Related Models\label{sec:related-models}}
In prior work \citet{nietert2022outlier,nietert2023outlier}, outlier robustness in addition to Wasserstein distributional robustness was handled by  introducing a new distance measure
\[W_{p,\epsilon}(\sP_1, \sP_2) = \inf_{\TV(\sP', \sP_1) \le \epsilon} W_p(\sP', \sP_2) = \inf_{\TV(\sP', \sP_2) \le \epsilon} W_p(\sP', \sP_1),\]
where \(W_p\) is the classical Wasserstein metric,  with the problem definition 
\begin{equation}\label{eq:or-wdro}
  \min_{\vw \in \R^d} \sup_{\sQ: W_{p,\epsilon}(\ppcorrupt, \sQ) \le \rho}\E_{\vxi \sim \sQ}[\ell(\vw; \vxi)]. \tag{OR-WDRO}
\end{equation}

Through the lens of the decision model \eqref{eq:decision}, we believe that \eqref{eq:or-wdro} is best used to model the scenario where the data generating and the training distributions are the same ($\ppcorrupt = \sP_0$), but the testing distribution $\pptest$ could have outliers in addition to a distributional shift from $\sP_0$. 

We remark that \eqref{eq:or-wdro} seems to be able to also model errors from training data contamination (i.e., Case (iii)) in the following sense:
suppose the training data are corrupted so that $\TV(\ppcorrupt, \sP_0) \le \epsilon$; then $\sP_0 \in \{\sQ: W_{p,\epsilon}(\ppcorrupt, \sQ)= 0\}$. If, in addition, we allow distribution shifts and assume that $W_p(\pptest, \sP_0) \le \rho$, then $W_{p,\epsilon}(\ppcorrupt, \pptest ) \le \rho$. In this sense, \eqref{eq:or-wdro} minimizes the worst-case loss with respect to an ambiguity set that contains \(\pptest\). However, we argue that modeling data contamination with \eqref{eq:or-wdro} has multiple drawbacks:

\begin{enumerate}[leftmargin=15pt, topsep=0pt, itemsep=-2pt]
  \item Dimension dependence: the excess risk bound of the solution of \eqref{eq:or-wdro} depends on the dimension of the data $\vxi$ (or the dimension of the underlying features), which can be \new{a significant factor} in high-dimensional settings~\cite[Theorems 1 and 4]{nietert2023outlier}. 
  \item \new{For the canonical task of robust mean estimation, the \eqref{eq:or-wdro} objective becomes ill-posed (evaluates to infinity) for Wasserstein-$p$ metrics with $p < 2$.}
\item We prove that the limit behavior of \eqref{eq:or-wdro} as $\rho \to 0$ is a variant of Conditional Value at Risk (CVaR)\footnote{The limit behavior of the tractable formulation in \citet[Section 3.3]{nietert2023outlier} is exactly CVaR under limits; see \Cref{lemma:limit-tractable-ordro}.}, which is known to be sensitive to outliers; see Theorem 2.1 and the discussion thereafter in~\citet{blanchet2024distributionally}, noting that CVaR can be interpreted as the limit of some $f$-divergence DRO risks~\cite[Example 3]{duchi2021learning}. In other words, a model parameterized by the optimal solution to \eqref{eq:or-wdro} is not robust to outliers in the training set, even without any distributional shifts\footnote{In fact, the maximizing distribution $\argmax_{\sQ} \big\{\E_{\vxi \sim \sQ} [ \ell(\vw; \vxi)]: \exists \sP', \epsilon' \le \epsilon \text{ s.t. } \ppcorrupt = (1 - \epsilon') \sQ + \epsilon' \sP'\big\}$ puts weights only on $(1-\eps)$-fraction of samples with \emph{largest} loss values, so in this sense CVaR-DRO is the opposite of robust estimation with data contamination.}.
\end{enumerate}

\new{The first implication follows directly from the cited work. We now elaborate on the latter two points with detailed technical arguments.} 
We begin with a standard fact (folklore; proof provided in \Cref{sec:proof-wasserstein-loss-moment}) that we will invoke repeatedly.
\begin{fact}\label{fact:wasserstein-loss-moment}
Let $(\mX, Y)$ be a random variable pair with distribution $\sP_0$ on $\R^d \times \R$. Assume that the marginal distribution of $\mX$ under $\sP_0$, denoted $\sP_{0}^{\mX}$, is a continuous distribution and has a finite $p$-th moment, i.e., $\E_{\sP_0}[\norm{\mX}^p] < \infty$, for some $p \ge 1$.
Let $\vw \in \mathbb{R}^d$ be a non-zero vector and $\rho > 0$ be a given positive constant.
Suppose that for each $y \in \R$, the loss function $\ell_y: \mathbb{R} \to \mathbb{R}$ is non-negative and satisfies:
$$\lim_{\abs{z} \to \infty} \frac{\ell_y(z)}{\abs{z}^{p+\iota}} = C$$
for some constants $\iota > 0$ and $C > 0$ that are independent of $y$.

Then,
$\sup_{\sQ: W_p(\sP_0, \sQ) \le \rho} \E_{(\mX',Y') \sim \sQ}[\ell_{Y'}(\vw \cdot \mX')] = +\infty,$
where the Wasserstein distance $W_p$ is defined with respect to the cost function $c((\vx_1,y_1),(\vx_2,y_2)) = \norm{\vx_1-\vx_2} + \kappa \abs{y_1 - y_2} $ for some $\kappa \in [0, +\infty]$.
\end{fact}

We illustrate the second drawback through a canonical problem in robust statistics---robust mean estimation. In this setting, applying the \eqref{eq:or-wdro} formulation in an attempt to directly account for outliers or corrupted samples in the training data often leads, under common conditions, to objectives that are uninformative or ill-posed.

\begin{example}[Mean Estimation and OR-WDRO for Training Data Contamination]\label{example:rme}
Consider the fundamental task of robust mean estimation, where the loss function is \(\ell(\vw; \vx) = \norm{\vw - \vx}_{2}^{2}\). Let \(\sP_0\) be the true, uncontaminated data distribution. 

Suppose we have training data from \(\ppcorrupt\),  which contains both ``local'' Wasserstein-$1$ perturbation and ``global'' TV contamination, i.e., there exists a distribution \(\sP'\) with $W_p(\sP', \sP_0) \le \rho$ and $\TV(\sP', \ppcorrupt) \le \epsilon$. To use \eqref{eq:or-wdro} formulation to define an ambiguity set of radius $\rho$ around $\ppcorrupt$ that is intended to contain $\sP_0$, we set $W_{p,\epsilon}(\ppcorrupt, \sP_0) \le \rho$. Our goal in this example is to illustrate that such a modeling approach for training data contamination using OR-WDRO can be problematic for $p<2$. Therefore, for this example, we set the final test distribution of interest to be $\pptest = \sP_0$, as our aim is to scrutinize how OR-WDRO might model training data contamination that separates an observed corrupted distribution $\ppcorrupt$ from $\sP_0$. The OR-WDRO objective is 
\begin{equation}\label{eq:or-dro-mean-estimation}
  \min_{\vw} \sup_{\sQ: W_{p,\epsilon}(\ppcorrupt, \sQ) \le \rho} \E_{\vx \sim \sQ}[\ell(\vw; \vx)],
\end{equation}
which equals infinity when $p<2$ as a consequence of \Cref{fact:wasserstein-loss-moment}; see the proof and intuition below. 

This highlights a modeling limitation of OR-WDRO when interpreted as a direct mechanism for handling training data contamination, because methods such as those discussed in \citet{nietert2024robust} and \citet[Theorem 1.4]{pittas2024optimal} can provide finite error guarantees (\(O(\sigma\sqrt\epsilon + \rho)\)) for mean estimation when dealing with $\epsilon$-corrupted samples from some $\tilde{\sP}$ such that $W_1(\tilde{\sP}, \sP_0) \le \rho$.
\end{example}

We now provide proof that \Cref{eq:or-dro-mean-estimation} equals $+\infty$ when $p<2$.
\begin{proof}
The ambiguity set $\cU = \{\sQ: W_{p,\epsilon}(\ppcorrupt, \sQ) \le \rho\}$ can be expanded using the definition of $W_{p,\epsilon}(\ppcorrupt, \sQ) = \inf_{\TV(\sP', \ppcorrupt) \le \epsilon} W_p(\sP', \sQ)$. Thus, $\cU = \{\sQ: \exists \sP' \text{ s.t. } \TV(\sP', \ppcorrupt) \le \epsilon \text{ and } W_p(\sP', \sQ) \le \rho\}$. This means that the ambiguity set $\cU$ is the union of all $W_p$-balls of radius $\rho$ centered at every distribution $\sP'$ within the $\epsilon$-TV ball around $\ppcorrupt$, i.e., $\cU = \bigcup_{\sP': \TV(\sP', \ppcorrupt) \le \epsilon} B_{W_p}(\sP', \rho)$.
The OR-WDRO objective can then be expressed as:
\[ \min_{\vw \in \R^d} \sup_{\sP': \TV(\sP', \ppcorrupt) \le \epsilon} \left( \sup_{\sQ: W_p(\sP', \sQ) \le \rho} \E_{\vxi \sim \sQ}[\ell(\vw; \vxi)] \right). \]
Now consider any distribution $\sP'$ such that $\TV(\sP', \ppcorrupt) \le \epsilon$. We know that $\sP'$ can have unbounded support because $\TV$ contamination is blind to geometry. Thus the inner supremum, $\sup_{\sQ: W_p(\sP', \sQ) \le \rho} \E_{\vxi \sim \sQ}[\norm{\vw - \vxi}_{2}^{2}]$, will be infinite for any $\vw \in \R^d$ when $p<2$. This is a standard result (\Cref{fact:wasserstein-loss-moment}) for $W_p$-DRO with $p<2$ when the loss function (here, quadratic) grows faster than the $p$-th power of the transport cost; adversarial perturbations $\sQ$ of $\sP'$ can shift mass to points $\vx$ with arbitrarily large $\norm{\vx}_2$ at a bounded $W_p$ cost, driving the expected loss to infinity.
\end{proof}

We next analyze the limiting behavior of the OR-WDRO objective when the Wasserstein radius $\rho$ approaches zero. This analysis highlights an inherent connection to Conditional Value-at-Risk (CVaR)-like measures. Given that CVaR is known to be sensitive to extreme data points\footnote{See Theorem 2.1 and the discussion thereafter in~\citet{blanchet2024distributionally}; CVaR can be interpreted as the limit of some $f$-divergence DRO risks~\cite[Example 3]{duchi2021learning}.}, this connection underscores why the OR-WDRO model may not achieve robust protection against training data outliers in the way dedicated robust statistics methods do.

\begin{lemma}[Limit Behavior of OR-WDRO Formulation]\label{lemma:or-wdro-limit-cvar-variant}
For any $\vw \in \R^d$, \begin{align}
  \sup_{\sQ: W_{p,\epsilon}(\ppcorrupt, \sQ) = 0} \E_{\vxi \sim \sQ}[\ell(\vw; \vxi)] & = (1-\epsilon) \CVaR_{\vxi \sim \ppcorrupt}^{1-\epsilon} [\ell(\vw; \vxi)] + \epsilon \sup_{\vxi} \ell(\vw; \vxi), \label{eq:or-wdro-cvar-variant}
\end{align}
where $\CVaR_{\vxi \sim \ppcorrupt}^{1-\epsilon} \ell(\vw; \vxi)$ is defined as
\begin{equation}\label{eq:cvar-definition}
  \sup_{\sP'} \big\{\E_{\sP'} [ \ell(\vw; \vxi)]: \exists \text{ probability distribution } \tilde\sP, \epsilon' \le \epsilon \text{ s.t. } \ppcorrupt = (1 - \epsilon') \sP' + \epsilon' \tilde\sP\big\}.
\end{equation}
\end{lemma}
\begin{proof}
The condition $W_{p,\epsilon}(\ppcorrupt, \sQ) = 0$ means $\inf_{\TV(\sP', \ppcorrupt) \le \epsilon} W_p(\sP', \sQ) = 0$. Because the domain of $\xi$ is complete and separable (therefore Polish, and thus $W_p(\sP', \sQ)=0 \implies \sP' = \sQ$), this implies that $\sQ$ is in the $W_p$-closure of the $\epsilon$-TV ball around $\ppcorrupt$, which for practical purposes means $\TV(\ppcorrupt, \sQ) \le \epsilon$.
Thus, the left-hand side of \eqref{eq:or-wdro-cvar-variant} becomes $\sup_{\sQ: \TV(\ppcorrupt, \sQ) \le \epsilon} \E_{\vxi \sim \sQ}[\ell(\vw; \vxi)]$. 

By \citet[Theorem 2.4]{bennouna2022holistic}, it holds that \[\sup_{\sQ: \LP_{\{\vzero\}}(\ppcorrupt, \sQ) \le \epsilon} \E_{\vxi \sim \sQ}[\ell(\vw; \vxi)]  = (1-\epsilon) \CVaR_{\vxi \sim \ppcorrupt}^{1-\epsilon} [\ell(\vw; \vxi)] + \epsilon \sup_{\vxi} \ell(\vw; \vxi), \] 
where $\LP_{\{\vzero\}}(\sP, \sQ)$ is defined to be $\inf_{\pi \in \Pi(\sP, \sP')} \int \Ind_{\vxi - \vxi' \not\in \{\vzero\}} \dd \pi(\xi, \xi') $ and $\Pi(\sP, \sP')$ denotes the set of joint distributions whose marginal distributions are $\sP$ and $\sP'$; \citet{strassen1965existence} shows that $\LP_{\{\vzero\}}(\sP, \sP') = \TV(\sP, \sP')$. 
\end{proof}

For tractable computation, \citet{nietert2023outlier} introduce a tractable variant of \eqref{eq:or-wdro} using a one-sided discrepancy $W_{p,\epsilon}(\ppcorrupt \| \sQ) = \inf_{\sP': \sP' \le \frac1{1-\epsilon} \ppcorrupt} W_p(\sP', \sQ)$ (where $\sP'$ is a probability measure) and imposing a moment constraint $\sQ \in \cG_2(\sigma, \vz_0) := \{\pi:\E_{\vxi \sim \pi}[\norm{\vxi - \vxi_0}_2^2] \le \sigma^2 \}$. Their tractable objective is:
\[\min_\vw \sup_{\sQ \in \cG_2(\sigma, \vz_0), W_{p,\epsilon}(\ppcorrupt \| \sQ) \le \rho} \E_{\vxi \sim \sQ}[\ell(\vw; \vxi)].\]
The structure of $W_{p,\epsilon}(\ppcorrupt \| \sQ)$ implies that $\sQ$ is effectively compared against \emph{parts} of $\ppcorrupt$, denoted by $\sP'$, which are consistent with the output of a deletion-only contamination model where $\ppcorrupt = (1-\epsilon') \sP' + \epsilon' \tilde\sP$ for some $0 \le \epsilon' \le \epsilon$. Note the way it connects to the CVaR definition \eqref{eq:cvar-definition}; we interpret this connection as the core intuition behind the following result. 

\begin{lemma}[Limit Behavior of Tractable OR-WDRO Reformulation]\label{lemma:limit-tractable-ordro}
In the limit where $\sigma \to +\infty$ (e.g., we have limited information about the reference distribution $\sP_0$) and the Wasserstein radius $\rho = 0$, the tractable OR-WDRO objective converges to the CVaR term:
\[ \lim_{\sigma\to\infty, \rho=0} \left( \sup_{\substack{\sQ \in \cG_2(\sigma, \vz_0),\\ W_{p,\epsilon}(\ppcorrupt \| \sQ) \le \rho}} \E_{\vxi \sim \sQ}[\ell(\vw; \vxi)] \right) = \CVaR_{\vxi \sim \ppcorrupt}^{1-\epsilon} [\ell(\vw; \vxi)]. \]
\end{lemma}
\begin{proof}
Based on the dual formulation in \citet[Remark 3]{nietert2023outlier}:
\begin{align*}
 & \sup_{\substack{\sQ \in \cG_2(\sigma, \vz_0),\\ W_{p,\epsilon}(\ppcorrupt \| \sQ) \le \rho}} \E_{\vxi \sim \sQ}[\ell(\vw; \vxi)] \\
 & = \; \inf_{\lambda_1, \lambda_2 \ge 0} \left( \lambda_1 \sigma^2 + \lambda_2 \rho^p + \CVaR_{\vxi \sim \ppcorrupt}^{1-\epsilon} \left[\sup_{\vxi'} (\ell(\vw; \vxi') - \lambda_1 \norm{\vxi' - \vz_0}^2 - \lambda_2 \norm{\vxi' - \vxi}^p)\right] \right).
\end{align*}
When $\rho = 0$, the term $\lambda_2 \rho^p$ is zero. The structure of the objective implies that to minimize the expression with respect to $\lambda_2$, effectively $\lambda_2$ is driven so large that $\vxi'=\vxi$ becomes the optimal choice in the innermost supremum. This simplifies the expression to:
\[ \sup_{\sQ: W_{p,\epsilon}(\ppcorrupt \| \sQ) = 0} \E_{\vxi \sim \sQ}[\ell(\vw; \vxi)] = \inf_{\lambda_1 \ge 0} \left( \lambda_1 \sigma^2 + \CVaR_{\vxi \sim \ppcorrupt}^{1-\epsilon} [\ell(\vw; \vxi) - \lambda_1 \norm{\vxi - \vz_0}^2 ] \right). \]
As $\sigma \to +\infty$, the term $\lambda_1 \sigma^2$ forces the optimal $\lambda_1$ to be $0$ to maintain a finite infimum. Thus, the expression converges to $\CVaR_{\vxi \sim \ppcorrupt}^{1-\epsilon} [\ell(\vw; \vxi)]$.
\end{proof}

These limiting behaviors (\Cref{lemma:or-wdro-limit-cvar-variant,lemma:limit-tractable-ordro}) are revealing. As \citet{nietert2023outlier} does not characterize the worst-case distributions for \eqref{eq:or-wdro} in a way that demonstrates reduced weighting of outliers in $\ppcorrupt$, the convergence to CVaR-like measures suggests that OR-WDRO may inherently assign significant weight to data points that yield large losses. This is a characteristic of CVaR as previously discussed and contrasts with the goal of robust statistics to downweight or ignore outliers when learning from contaminated training data, especially if there is little prior knowledge about the covariance of the reference distribution $\sP_0$ (thus it makes sense to set $\sigma \to \infty$). It is also clear that the limit behavior of \eqref{eq:or-wdro} does not recover the robust statistics formulation of data contamination.

The intuition that OR-WDRO is not ideally suited for handling TV-contamination in the training set (a weaker model than the $\epsilon$-strong contamination considered in our work) can be further explored by setting $\rho=0$. In this case, the OR-WDRO problem simplifies to \[\min_{\vw} \max_{\sQ : \TV(\sQ, \ppcorrupt) \le \epsilon} \E_{\vxi \sim \sQ}[\ell(\vw; \vxi)].\] If the true clean distribution $\sP_0$ is within this TV-ball, the objective robustifies against the worst case within this ball, which leads to the CVaR behavior. This approach, while providing a guarantee over the TV-ball, does not necessarily leverage more specific structural information about $\sP_0$ beyond its containment. Although \citet{nietert2023outlier} use structural assumptions like moment constraints for their tractable variant, these are aids for computation rather than intrinsic parts of the original definition \eqref{eq:or-wdro}. Similar TV-robust formulations have also been noted for their sensitivity to training data outliers~\cite[Remark 2.5]{bennouna2022holistic}
\footnote{Remark 2.5 in~\citet{bennouna2022holistic} distinguishes their LP-DRO approach (which can amplify the influence of high-loss data points in training data) from traditional robust statistics, noting that the latter's strategy of outlier identification and suppression often relies on structural assumptions about the true distribution $\sP_0$. They position their work, which makes no such assumptions about $\sP_0$, in contrast to~\citet{nietert2023outlier}, implying the latter operates under, or would require, such structural assumptions to effectively handle statistical outliers. We observe, however, that the fundamental OR-WDRO formulation \eqref{eq:or-wdro} from~\citet{nietert2023outlier} does not explicitly incorporate these structural conditions on $\sP_0$ into its definition. Consequently, as our analysis indicates (e.g., \Cref{lemma:or-wdro-limit-cvar-variant}), the OR-WDRO formulation \eqref{eq:or-wdro} can also exhibit sensitivity to training data outliers, a characteristic effectively shared by the LP-DRO approach described in~\citet{bennouna2022holistic}.}.

In conclusion, while the OR-WDRO framework \eqref{eq:or-wdro} and its variants are valuable for scenarios involving potential outliers or misspecifications in the \emph{test} distribution---a common concern in fields like finance or generative modeling where robust out-of-sample performance is critical---its properties, particularly its limiting connections to CVaR, suggest it is not optimally designed for mitigating the impact of adversarial outliers present within the \emph{training data}. This distinction motivates our paper's alternative approach, which separates the mechanism for handling training data contamination from that for ensuring robustness to test-time distributional shifts.

The summary of most closely related models and how they relate to our  model~\eqref{eq:decision} is in Table~\ref{tab:table-comparison}.
\begin{figure}[ht]
  \centering
  \begin{minipage}{\linewidth}
    \centering
    \captionof{table}{Summary comparison of the data contamination and distributional shift assumptions.\looseness=-1}
    \resizebox{\textwidth}{!}{\label{tab:table-comparison}
  \begin{tabular}{  c  c  c }
  \toprule
     \textbf{Training data} & \textbf{Testing distribution} \(\pptest\) & \textbf{Reference} \\
    \midrule
    \(\ppcorrupt = \sP_0\) &  \(W_{p,\epsilon}(\pptest, \sP_0) \le \rho \) & OR-WDRO~\citet{nietert2023outlier}\\
     \(W_{p,\epsilon}(\ppcorrupt, \sP_0 ) \le \rho \)\footnotemark &  \(\pptest = \sP_0\) & Local \& global corruption~\citet{nietert2024robust,pittas2024optimal} \\
     $\epsilon$-corrupted data from $\sP_0$ & \(W_{p}(\sP_0, \pptest) \le \rho\) &
     \textbf{Our work} \\
     \bottomrule
  \end{tabular}
  }
\end{minipage}
\end{figure}
\footnotetext{Training data in~\citet{pittas2024optimal} can simultaneously be $\eps$-corrupted and have Wasserstein perturbation.}

Another approach that, like ours, simultaneously addresses both training data contamination (modeled via $\TV(\sP_0, \ppcorrupt) \le \eps$) and distributional shifts at test time is the Distributionally Robust Outlier-aware Optimization (DORO) framework proposed by \citet{zhai2021doro}. 
DORO defines its objective, coined as \emph{DORO risk}, as the minimum $f$-divergence-based DRO risk (with shift parameter $D_f(\pptest, \ppcorrupt) \le \rho$) among all possible $(1-\epsilon)$-subpopulations of the training data. 
While the authors provide theoretical guarantees for the DORO risk minimizer, their proposed algorithm lacks both error bounds and computational complexity guarantees; as such, it remains unclear whether the DORO risk minimizer can be approximated provably---let alone efficiently; see further discussion and a counterexample in \Cref{sec:doro-appendix}.

The work by \citet{wang2024outlier} introduces a DRO framework defined by a novel ambiguity measure inspired by Unbalanced Optimal Transport (UOT). They propose that UOT-based distance is inherently more resilient to outliers and their formulation subtracts a penalty term based on prior knowledge to further down-weight contaminated data. While \citet{wang2024outlier} used this approach to handle contaminated training data and provided empirical results demonstrating improved robustness in such settings compared to standard Wasserstein DRO and OR-WDRO, they did not establish explicit theoretical guarantees on estimation error. More importantly, from the modeling perspective of \eqref{eq:decision}, its mechanism of defining an ambiguity suggests that the outliers it robustifies against are perhaps best characterized in the testing distribution ($\pptest$), as discussed previously for~\citet{nietert2023outlier} and in \Cref{sec:modeling-principles}. 

In the sequel, we use data \emph{contamination} or \emph{corruptions} to specifically indicate the existence of outliers in the \emph{training data}, whereas \emph{outliers} might exist in either training or testing data.

\subsection{Equivalent DRO Reformulation\label{sec:dro-reg}}
In this work, we focus on convex Lipschitz-continuous functions that can be written as a composition of a high-dimensional affine function and a univariate nonlinear function. In particular, if for labeled data $(\vx, y)$ the loss function $\ell(\vw; \vx, y)$ takes the form of either $\phi_1(\vw \cdot \vx - y)$\footnote{More generally, for loss functions of the form $\phi_1(\vw \cdot \vx - y)$, we can also allow for distributional shifts in the label $\thelabel$: when \(c((\vx, \thelabel), (\vx', \thelabel')) = \norm{[\vx^{\top}, \thelabel]^{\top} - [\vx'^{\top}, \thelabel']^{\top}}_r\), we replace the regularization $\norm{\vw}_s$ by \(\norm{[\vw^{\top}, -1]^{\top}}_s\). For conciseness, in this paper we only consider distributional shifts in the covariates $\vx$.} or $\phi_2(y \vw \cdot \vx)$ for univariate functions $\phi_1$ and $\phi_2$ of Lipschitz modulus $\zeta$ and the cost function is $c((\vx, \thelabel), (\vx', \thelabel')) = \norm{\vx - \vx'}_r + \chi(\thelabel - \thelabel')$ for some $r \in [1, +\infty]$, where $\chi(0) = 0$ and is $+\infty$ elsewhere, then \citet[Theorem 4 and Remark 18]{shafieezadeh2019regularization} implies that \eqref{eq:w1-dro} is equivalent to 
\begin{equation}\label{eq:w1-dro-regularization}
  \min_\vw \E_{\sP_0} [\ell(\vw; \vx, y)] + \rho \zeta \norm{\vw}_s,
\end{equation}
where $1/r + 1/s = 1.$ Thus, our problem reduces to solving \eqref{eq:w1-dro-regularization} with sample access to data-contaminated distribution $\ppcorrupt$ (cf.\ \eqref{eq:decision}). 
Due to standard uniform convergence results, it suffices to robustly solve an empirical version of \eqref{eq:w1-dro-regularization} as in problem \eqref{eq:optimization_formulation}.

We give a few examples of reformulating \Cref{eq:w1-dro} as \Cref{eq:w1-dro-regularization}. Recall that the cost function is $c((\vx, \thelabel), (\vx', \thelabel')) = \norm{\vx - \vx'}_r + \chi(\thelabel - \thelabel')$; in the following, we drop the superscript $c$, since the context is clear. The convex conjugate of these loss functions and their efficient proximal operator evaluation per \Cref{def:proximable} can be found in \citet[Sections 4.4.16 and 6.9]{beck2017first}.

\begin{example}[$W_1$-robust regression]\label{example:W1-robust-regression}
  For $\rho > 0$ and the following loss functions $\ell$, it holds that
  \[\sup_{\sQ: W_{1}(\sQ, \sP_0) \le \rho}\E_{(\vx, \thelabel) \sim \sQ}[\ell(\vw; \vx, \thelabel)] = \E_{(\vx, \thelabel) \sim \sP_0}[\ell(\vw; \vx, \thelabel)] + \rho  \norm{\vw}_s:\]
  \begin{itemize}[nosep, leftmargin=15pt, topsep=0pt, itemsep=-2pt]
    \item {Least absolute deviation} $\ell(\vw; \vx, \thelabel) = \abs{\vw \cdot \vx - \thelabel}$,
    \item {Huber regression}   $\ell(\vw; \vx, \thelabel) = h(\vw \cdot \vx - \thelabel)$, where $h(t) = \begin{cases} t^2/2 & {\abs{t} < 1} \\ \abs{t}  & {\abs{t} \ge 1}\end{cases}$.
  \end{itemize}
\end{example}

\begin{example}[$W_1$-robust classification]\label{example:w1-robust-classification}
  We restrict the labels $y \in \{+1, -1\}$. For $\rho > 0$ and the following $\ell$, it holds that 
  $\sup_{\sQ: W_{1}(\sQ, \sP_0) \le \rho}\E_{(\vx, \thelabel) \sim \sQ}[\ell(\vw; \vx, \thelabel)] = \E_{(\vx, \thelabel) \sim \sP_0}[\ell(\vw; \vx, \thelabel)] + \rho  \norm{\vw}_s$:
  \begin{itemize}[nosep, leftmargin=15pt, topsep=0pt, itemsep=-2pt]
    \item {Hinge loss/ SVM classification}: $\ell(\vw; \vx, \thelabel) = (1 - \thelabel \vw \cdot \vx)_+$,
    \item {Logistic regression}: $\ell(\vw; \vx, \thelabel) = \log(1 + \exp(-\thelabel \vw \cdot \vx ))$.
  \end{itemize}
\end{example}

\new{Our results directly apply to both \Cref{example:W1-robust-regression} and \Cref{example:w1-robust-classification}, thanks to the finite-sum structure of their empirical loss functions. Below, we also include the example of $W_2$-robust linear regression, which does not exhibit the same finite-sum loss structure (notice the square-root in \eqref{eq:w2-linear-regression}), but to which our results apply after a simple transformation; see \Cref{remark:W2-LR-extension}.} 

\begin{example}[$W_2$-robust linear regression]\label{example:W2-robust-LR}
  For $\rho > 0$, it holds that
  \begin{equation}\label{eq:w2-linear-regression}
  	\sup_{\sQ: W_{2}(\sQ, \sP_0) \le \rho}\E_{(\vx, \thelabel) \sim \sQ}[(\vw \cdot \vx - \thelabel)^2] = \sqrt{\E_{(\vx, \thelabel) \sim \sP_0}[(\vw \cdot \vx - \thelabel)^2]} + \rho \norm{\vw}_s.
  \end{equation}
\end{example}

\subsection{Inexact Hybrid Gradient Oracle\label{sec:inexact-hg-oracle}}

Finally, we argue that under standard distributional assumptions (stated as \Cref{assump:dist-assumption}), we can assume access to a bounded-error oracle for data drawn from $\sP_0.$ 

\begin{assumption}[Distributional Assumptions of $\sP_0$]\label{assump:dist-assumption}
Let $\{\tilde\vx_i\}_{i=1}^N$ be a set of $N$ independent covariates drawn from a distribution over $\R^{d-1}$ with mean $\tilde \vmu = \vzero$ and covariance matrix $\tilde\mSigma$ with $\opnorm{\tilde\mSigma} \le \sigma^2$. We define $\vx_i = [1, \tilde \vx_i^{\top}]^{\top}$ by prepending $1$ to $\tilde \vx_i$ for each $i \in [N]$. We make no assumptions about the label $\{\thelabel_i\}$.  
\end{assumption}

The stipulation in \Cref{assump:dist-assumption} that covariates $\tilde{\vx}$ (i.e., $\vx$ excluding its first, constant component) are centered, meaning $\E[\tilde{\vx}] = \vzero$, is a standard simplification that can be made without loss of generality for linear prediction models incorporating an intercept term $w^0$. Indeed, for a fixed $\tilde \vw \in \R^{d-1}$,  $\tilde \vw \cdot \tilde \vx + w^0 = \tilde \vw \cdot (\tilde \vx - \E[\tilde \vx]) + w^0 + \tilde\vw \cdot \E[\tilde \vx]$. This shows that a model prediction based on potentially uncentered covariates $\tilde \vx$ with an intercept $w^0$ is identical to a prediction using centered covariates $(\tilde{\vx} - \E[\tilde{\vx}])$ if we use the exact same weight $\tilde{\vw}$ but adjust the intercept to be $(w^0 + \tilde{\vw} \cdot \E[\tilde{\vx}])$. The mean $\mathbb{E}[\tilde{\mathbf{x}}]$ can be estimated from $\epsilon$-corrupted data using \Cref{alg:robust_mean_estimation}. See \Cref{sec:arbitrary-mean} for details. 
We verify in \Cref{sec:prepending-by-one} that prepending covariates by $1$ preserves the bounded covariance assumption.

We prove the following key technical lemma that enables the robust estimation of inexact hybrid gradients for \emph{all} iterations, which is a nontrivial task because iterates depend on previous iterations.

\begin{lemma}[Stable set under unknown bounded scaling is stable]
\label{lemma:bounded-sequence-scaling-stability}
  Let $S' = \{\vx_i\}_{i=1}^N$ be an $(\epsilon, \delta)$- stable set with respect to a mean $\vmu \in \R^d$ and variance $\sigma^2$. Given an arbitrary bounded sequence \((\extradual_i)_{i=1}^N\) with \(\abs{\extradual_i} \le \zeta\; (i \in [N])\), the set \(\{\extradual_i \vx_i\}_{i=1}^N\) is $(\epsilon, O(\sqrt\epsilon))$-stable with respect to its empirical mean $\frac1N \sum_i \extradual_i \vx_i$ and variance $\zeta^2 (\sigma^2 + \norm{\vmu}_2^2)$. 
\end{lemma} 
\begin{proof}
Let \(\vz_i = \extradual_i \vx_i\) for each \(i \in [N]\).
Define the empirical mean of the scaled set as
\(
\vmu_\vz = \frac{1}{N} \sum_{i=1}^N \vz_i
\),
and its second moment matrix as
\(
\mM_\vz = \frac{1}{N} \sum_{i=1}^N \vz_i \vz_i^\top
\).
The empirical covariance matrix of the scaled set is
\(
\Cov(\{\vz_i\}) = \mM_\vz - \vmu_\vz \vmu_\vz^\top
\).   
Since \(\vmu_\vz \vmu_\vz^\top\) is Positive Semi-Definite (PSD), we have \(\Cov(\{\vz_i\}) \preceq \mM_\vz\), which implies
\(    \opnorm{\Cov(\{\vz_i\})} \le \opnorm{\mM_\vz}    \).

We compute that \(\mM_\vz = \frac{1}{N} \sum_{i=1}^N \vz_i \vz_i^\top = \frac{1}{N} \sum_{i=1}^N (\extradual_i \vx_i)(\extradual_i \vx_i)^\top = \frac{1}{N} \sum_{i=1}^N \extradual_i^2 \vx_i \vx_i^\top\).    
By the assumption \(\abs{\extradual_i} \le \zeta\), we have \(\extradual_i^2 \le \zeta^2\) for all \(i\).
Since \(\vx_i \vx_i^\top\) is PSD, it follows that
\(
\extradual_i^2 \vx_i \vx_i^\top \preceq \zeta^2 \vx_i \vx_i^\top
\).
Summing over \(i\) gives
\(
\mM_\vz \preceq \zeta^2 \left( \frac{1}{N} \sum_{i=1}^N \vx_i \vx_i^\top \right) = \zeta^2 \mM_\vx
\),
where \(\mM_\vx = \frac{1}{N} \sum_{i=1}^N \vx_i \vx_i^\top\) is the second moment matrix of the original set.
Taking operator norms, we obtain
\(
\opnorm{\mM_\vz} \le \zeta^2 \opnorm{\mM_\vx}
\).

Now, we relate \(\mM_\vx\) to the covariance matrix \(\Cov(\{\vx_i\})\) of the original set:
\(
\mM_\vx = \Cov(\{\vx_i\}) + \vmu \vmu^\top
\).
Thus,
\(    \opnorm{\mM_\vx} \le \opnorm{\Cov(\{\vx_i\})} + \opnorm{\vmu \vmu^\top}
\).
Since \(\vmu \vmu^\top\) is a rank-one PSD matrix, \(\opnorm{\vmu \vmu^\top} = \norm{\vmu}_2^2\).
By \Cref{fact:stability_bounded_covariance}, we know \(\opnorm{\Cov(\{\vx_i\})} \le \sigma^2\), so
\(
\opnorm{\mM_\vx} \le \sigma^2 + \norm{\vmu}_2^2
\).

Combining the inequalities, we conclude that \(
\opnorm{\Cov(\{\vz_i\})} \le \opnorm{\mM_\vz} \le \zeta^2 (\sigma^2 + \norm{\vmu}_2^2)\). Another application of \Cref{fact:stability_bounded_covariance} completes the proof.   
\end{proof}
This lemma implies that we can use \Cref{lemma:bounded-sequence-scaling-stability} to construct an oracle that satisfies \Cref{assump:approximation-guarantee} via robust mean estimation algorithms described in Section \ref{sec:preliminaries}. 
\begin{proposition}\label{prop:inexact-hg-oracle}
  On input an arbitrary $3\zeta$-uniformly bounded sequence $(\extradual^i)_{i=1}^N$ and a $2\epsilon$-corrupted version of an $(O(\epsilon), O(\sqrt\epsilon))$-stable set $\{\vx_i\}_{i=1}^N$ with respect to variance $\sigma^2$, $\mathsf{RobustMeanEstimation}$ (\Cref{alg:robust_mean_estimation}) outputs $\hat{\vz}$ such that $\norm{\hat\vz - \frac1N \sum_{i=1}^{N} \extradual^i \vx_i}_2 = O(\sigma \zeta \sqrt\epsilon)$. 
\end{proposition}

\new{We remark on the parameter scaling in the error bound of \Cref{prop:inexact-hg-oracle}. Recall that the term $O(\sigma\sqrt{\epsilon})$ is the information-theoretically
optimal error for robustly estimating the mean of a
distribution with covariance bounded by $\sigma^2 I$ (even in one dimension); see \Cref{fact:lower_bound_instance} below for a precise statement. Our result shows that this error guarantee can be extended to the more complex setting of estimating a weighted mean of covariates; the additional factor of $\zeta$ appears as a result of the scaling. This oracle error of $\Delta = O(\zeta\sigma\sqrt{\epsilon})$ propagates to our final optimality gap, which \Cref{thm:lower_bound} proves to be information-theoretically optimal.}

\section{Algorithm and Analysis\label{sec:main-algo}}
In this section, we show how to solve convex Lipschitz data-contaminated $W_1$-DRO problems using an efficient primal-dual algorithm. Here we crucially rely on the problem reformulation described in the previous section to reduce the problem to minimizing functions of the form  $\frac1N \sum_{i=1}^{N} l_i(\vx_i \cdot \hat\vw) + \psi(\hat\vw)$ over parameter vectors $\vw,$ where $l_i$ is a univariate Lipschitz-continuous function and $\psi$ is proximable, addressing $W_1$-DRO examples from Section \ref{sec:dro-reg}. Due to the discussion from Section \ref{sec:inexact-hg-oracle}, effects of data contamination can be reduced to an inexact oracle $\cO(\vextradual, \{\vx_i\})$, as stated in  \Cref{assump:approximation-guarantee}.

For conciseness of notation, let $l_i \equiv \ell_{y_i}$ and let $g_i \equiv l_i^*$ be its convex conjugate.
The pseudocode for our algorithm is provided in \Cref{alg:main-opt}. 
\Cref{alg:main-opt}  can be seen as a variant of primal-dual hybrid gradient (PDHG) \cite{chambolle2011first}, but with two key differences addressed by our analysis: 
(1) the ``extrapolated gradient vector'' used in the computation of $\vw_k$ (see Line 6 in \Cref{alg:main-opt}) is accessed using a bounded-error oracle $\cO(\vextradual, \{\vx_i\})$, and (2) the amount of primal regularization (the quadratic term $\norm{\vu - \vw_{k-1}}_2^2$ in the definition of $\vw_k$) is scaled by $c_k$, which monotonically decreases over iterations $k$. The item
(2) is crucial for establishing an error guarantee that scales with $\|\vw_0 - \vw_*\|_2$ in place of the maximum distance of iterates to $\vw_*,$ which may be (much) larger. We also highlight that the inexact oracle is used only in the computation of primal updates, whereas the dual updates (in Line 8) are computed using corrupted samples.   
\begin{algorithm}[ht]
\textbf{Input} $\epsilon$, corrupted samples $\{(\vx_i^c, y_i^c)\}_{i=1}^N$, $\sigma$, $\vw_0 = \vzero, \norm{\vw_0 - \vw^*}_2$, estimation guarantee $\Delta$\\
\textbf{Initialization} $A_0 = a_0 = 0, c_0 = 2, \vdual_0 = \vdual_{-1} = \frac1N \vone$, $\gamma = \frac{\norm{\vw_0 - \vw^*}_2}{\zeta\sqrt{N}}$ \\
\For{$k = 1 \dots T:= \frac{2\zeta\sigma}\Delta$}{
  \(A_k = a_k + A_{k-1}, c_k = c_{k-1} - 1/T \), where $a_k = \sqrt{N} / \sigma$ \\
  \(\bar\vextradual_{k-1} = \bar\vdual_{k-1} + \frac{a_{k-1}}{a_k}(\bar\vdual_{k-1} - \bar\vdual_{k-2}) \) \\
  $\bar\vz_{k} = \mathsf{RobustMeanEstimation}(\{\bar\beta_{k-1}^i\vx_i^c\}_{i=1}^N, 2\epsilon)$ \\  
  $\vw_k = \argmin_{\vu \in \R^d} \{a_k  \bar\vz_k \cdot \vu + a_k \psi(\vu) + \frac{c_k}{2\gamma}\norm{\vu - \vw_{k-1}}_2^2\} \label{line:primal-update}$\\
  $\bar\dual_k^i = \argmax_{v^i \in \R} \{\frac{a_k}N ( v^i\vx_i^c \cdot \vw_k - g_i^c(v^i)) - \frac\gamma2 (v^i - \bar\dual^i_{k-1})^2\} \quad (\forall i \in [N])\label{line:dual-update} $ 
}
\textbf{Output} $\hat\vw_T = \frac1{A_T} \sum_{k=1}^T a_k \vw_k$\\
\caption{PDHG with inexact hybrid gradient and diminishing primal regularization\label{alg:main-opt}}
\end{algorithm}
\begin{algorithm}[ht]
\textbf{Input \& Initialization} 
\underline{\emph{stable}} samples $\{(\vx_i^s, y_i^s)\}_{i=1}^N$, same other parameters as in \Cref{alg:main-opt}\\
Compute $\opnorm{\mSigma_\mB}$, where $\mSigma_\mB$ is the empirical covariance matrix of $\{\vx_i^s\}_{i=1}^N$ \\
\For{$k = 1 \dots T:= \frac{2\zeta\opnorm{\mSigma_\mB}}{\Delta}$}{
  \(A_k = a_k + A_{k-1}, c_k = c_{k-1} - 1/T \), where $a_k = \sqrt{N} / \sigma$ \\
  \(\vextradual_{k-1} = \vdual_{k-1} + \frac{a_{k-1}}{a_k}(\vdual_{k-1} - \vdual_{k-2}) \) \\
  Compute arbitrary $\vz_{k}$ s.t. $\norm{\vz_k - \frac1N \sum_{i=1}^{N} \extradual_{k-1}^i \vx_i^s}_2 \le \Delta$\\  
  $\vw_k = \argmin_{\vu \in \R^d} \{a_k  \vz_k \cdot \vu + a_k \psi(\vu) + \frac{c_k}{2\gamma}\norm{\vu - \vw_{k-1}}_2^2\} $\\
  $\dual_k^i = \argmax_{v^i \in \R} \{\frac{a_k}N ( v^i\vx_i^s \cdot \vw_k - g_i^s(v^i)) - \frac\gamma2 (v^i - \dual^i_{k-1})^2\} \quad (\forall i \in [N])$ 
}
\textbf{Output} $\hat\vw_T = \frac1{A_T} \sum_{k=1}^T a_k \vw_k$\\
\caption{Idealized inexact PDHG with (unknown) stable set\label{alg:main-opt-analysis}}
\end{algorithm}

\begin{proposition}\label{prop:main-opt}
  Let $\vw^* \in \argmin_{\vw \in \R^{d}} \frac1N \sum_{i=1}^{N} \ell_{y_i}(\vx^s_i \cdot \vw) + \psi(\vw)$, where $\{\vx^s_i\}_{i=1}^N$ are the stable covariates. Under \Cref{assump:convex-lipschitz,assump:approximation-guarantee}, \Cref{alg:main-opt-analysis} outputs $\hat\vw_T$ that satisfies
  \[\frac1N \sum_{i=1}^{N} l_i(\vx_i^s \cdot \hat\vw_T) + \psi(\hat\vw_T) \le \frac1N \sum_{i=1}^{N} l_i(\vx_i^s \cdot \vw^*) + \psi(\vw^*) + 3 \norm{\vw^* - \vw_0}_2 \Delta.\]
\end{proposition}
\new{To prove this proposition, we repeatedly invoke the following fact for strongly convex functions.} \begin{fact}\label{fact:stron-convex-minimizer}
  Let $\phi: \R^k \to \R \cup \{+\infty\}$ be a $\mu$-strongly convex function and let $\vp^*$ be its minimizer. Then for all $\vp \in \R^k$, it holds that
 \(\phi(\vp^*) \le \phi(\vp) - \frac\mu2 \norm{\vp - \vp^*}_2^2. \)
\end{fact}
\begin{proof}[ of \Cref{prop:main-opt}]
\new{We first reformulate the objective using \Cref{def:convex_conjugate}.} Since
\begin{align*}
  \frac1N \sum_{i=1}^{N} \ell_i(\vx_i \cdot \vw) + \psi(\vw) 
  & = \frac1N \sum_{i=1}^{N} \max_{\dual^i \in \R} (\dual^i\vx_i \cdot \vw - l^*_i(\dual^i))  + \psi(\vw) \\
  & = \max_{\vdual \in \R^N}\frac1N \sum_{i=1}^{N}  (\dual^i\vx_i \cdot \vw - g_i(\dual^i)) + \psi(\vw), \numhere\label{eq:Fenchel}
\end{align*}
we define $L(\vw, \vdual) = \frac1N \sum_{i=1}^{N}  (\dual^i\vx_i \cdot \vw - g_i(\dual^i)) + \psi(\vw)$ and proceed to analyze \Cref{alg:main-opt} as a primal-dual algorithm.

First, we show an upper bound for $a_k L(\vw_k, \vv_{\star})$ for any given $ \vv_{\star}$ to be chosen later.

Since $\norm{\vv - \vdual}_2^2 = \sum_{i=1}^N (v^i - \dual^i)^2 $, we can separately write for each $i \in [N]$ that $\dual^i_k = \argmax_{v \in \R} \frac{a_k}N v\vx_i \cdot \vw_k - a_k g_i(v) -\frac\gamma2(v - \dual^i_{k-1})^2$. 
It holds that
\begin{align*}
  & \;a_k L(\vw_k, \vv_{\star}) 
  = \frac{a_k}N \sum_{i=1}^{N}  (v_\star^i\vx_i \cdot \vw_k - g_i(v_\star^i)) + a_k \psi(\vw_k) \\
\; &  = \frac{a_k}N \sum_{i=1}^{N}  (v^i_\star\vx_i \cdot \vw_k - g_i(v_\star^i)) + a_k \psi(\vw_k) -\frac\gamma2\norm{\vdual_{k-1} - \vv_{\star}}_2^2 +\frac\gamma2\norm{\vdual_{k-1} - \vv_{\star}}_2^2 \\
\;& \le \frac{a_k}N \sum_{i=1}^{N}  (\dual_k^i\vx_i \cdot \vw_k - g_i(\dual_k^i)) + a_k \psi(\vw_k) -\frac\gamma2\norm{\vdual_{k-1} - \vdual_k}_2^2 +\frac\gamma2\norm{\vdual_{k-1} - \vv_{\star}}_2^2 -\frac\gamma2\norm{\vdual_k - \vv_{\star}}_2^2 \\
\;& = a_k L(\vw_k, \vdual_k) -\frac\gamma2\norm{\vdual_{k-1} - \vdual_k}_2^2 +\frac\gamma2\norm{\vdual_{k-1} - \vv_{\star}}_2^2 -\frac\gamma2\norm{\vdual_k - \vv_{\star}}_2^2,
\end{align*} 
where the inequality follows from 1-strong convexity of $v \mapsto (1/2)(v - \dual_{k-1}^i)^2$ and \Cref{fact:stron-convex-minimizer}.

Next, we prove the lower bound for $a_k L(\vu_\star, \vdual_k)$.
Let $\ve_k = \cO(\vextradual_{k-1}, \{\vx_i\}) - \frac{1}N\sum_{i=1}^{N}  \extradual_{k-1}^i \vx_i$. It holds that
\begin{align*}
  & \; a_k L(\vu_\star, \vdual_k) 
   = \frac{a_k}N \sum_{i=1}^{N}  (\dual_k^i\vx_i \cdot \vu_\star - g_i(\dual_k^i) )+ a_k \psi(\vu_\star) \\
=  & \; \frac{a_k}N \sum_{i=1}^{N}  ((\dual_k^i - \extradual_{k-1}^i)\vx_i \cdot \vu_\star - g_i(\dual_k^i)) - \frac{c_k}{2\gamma}\norm{\vu_\star - \vw_{k-1}}_2^2 - a_k \ve_k \cdot \vu_\star \\
  & \; + \frac{a_k}N\sum_{i=1}^{N}  \extradual_{k-1}^i \vx_i \cdot \vu_\star + a_k \ve_k \cdot \vu_\star + a_k \psi(\vu_\star) + \frac{c_k}{2\gamma}\norm{\vu_\star - \vw_{k-1}}_2^2 \\
\ge  & \; \frac{a_k}N \sum_{i=1}^{N}  ((\dual_k^i - \extradual_{k-1}^i)\vx_i \cdot \vu_{\star} - g_i(\dual_k^i)) -  \frac{c_k}{2\gamma}\norm{\vu_{\star} - \vw_{k-1}}_2^2 - a_k \ve_k \cdot (\vu_{\star} -\vw_k) \\
  & \; + \frac{a_k}N\sum_{i=1}^{N}  \extradual_{k-1}^i \vx_i \cdot \vw_k + a_k \psi(\vw_k) + \frac{c_k}{2\gamma}\norm{\vw_k - \vw_{k-1}}_2^2 + \frac{c_k}{2\gamma}\norm{\vw_k - \vu_{\star}}_2^2, \numhere\label{eq:extra-regularization}
\end{align*}
where the last inequality follows from 1-strong convexity of $\vu \mapsto \frac{1}{2}\norm{\vu - \vw_{k-1}}^2_2$, the construction $\vw_k = \argmin_{\vu \in \R^d} \{\sum_{i=1}^{N} a_k  (\ve_k + \frac1N \extradual_{k-1}^i \vx_i) \cdot \vu + a_k \psi(\vu) + \frac{c_k}{2\gamma}\norm{\vu - \vw_{k-1}}_2^2\}$, and \Cref{fact:stron-convex-minimizer}. To continue the lower bound, note that $(c_k)$ is chosen to be a decreasing sequence and we rewrite \Cref{eq:extra-regularization} as follows: 
\begin{align*}
  & \; a_k L(\vu_\star, \vdual_k) \ge  \frac{a_k}N \sum_{i=1}^{N}  ((\dual_k^i - \extradual_{k-1}^i)\vx_i \cdot \vu_{\star} - g_i(\dual_k^i)) -  \frac{c_k}{2\gamma}\norm{\vu_{\star} - \vw_{k-1}}_2^2 - a_k \ve_k \cdot (\vu_{\star} -\vw_k) \\
  & \; + \frac{a_k}N\sum_{i=1}^{N}  \extradual_{k-1}^i \vx_i \cdot \vw_k + a_k \psi(\vw_k) + \frac{c_k}{2\gamma}\norm{\vw_k - \vw_{k-1}}_2^2 + \frac{c_{k+1} + (c_{k} - c_{k+1})}{2\gamma}\norm{\vw_k - \vu_{\star}}_2^2  \\
\ge & \; \frac{a_k}N \sum_{i=1}^{N} ( (\dual_k^i - \extradual_{k-1}^i)\vx_i \cdot \vu_{\star} - g_i(\dual_k^i)) -  \frac{c_k}{2\gamma}\norm{\vu_{\star} - \vw_{k-1}}_2^2 - \frac{a_k^2 \Delta^2\gamma}{2(c_{k} - c_{k+1})}\\
  & \; + \frac{a_k}N\sum_{i=1}^{N}  \extradual_{k-1}^i \vx_i \cdot \vw_k + a_k \psi(\vw_k) + \frac{c_k}{2\gamma}\norm{\vw_k - \vw_{k-1}}_2^2 + \frac{c_{k+1}}{2\gamma}\norm{\vw_k - \vu_{\star}}_2^2,
\end{align*}
where the last step uses Young's inequality to control the error term $a_k \ve_k \dot (\vu_{\star} - \vw_k) $.

Combining the lower bound and the upper bound, we obtain 
\begin{align*}
  & a_k (L(\vw_k, \vv_{\star}) - L(\vu_{\star}, \vdual_k)) \\
  \le & \; \frac{a_k}N \sum_{i=1}^{N}  (\dual_k^i - \extradual_{k-1}^i)\vx_i \cdot (\vw_k - \vu_{\star})   - \frac\gamma2 \norm{\vdual_{k-1} - \vdual_k}_2^2 + \frac\gamma2 \norm{\vdual_{k-1} - \vv_{\star}}_2^2 - \frac\gamma2 \norm{\vdual_k - \vv_{\star}}_2^2\\
  &\; +  \frac{c_{k}}{2}\norm{\vu_{\star} - \vw_{k-1}}_2^2
   - \frac{c_k}{2\gamma}\norm{\vw_k - \vw_{k-1}}_2^2 - \frac{c_{k+1}}{2\gamma}\norm{\vw_k - \vu_{\star}}_2^2 + \frac{a_k^2 \Delta^2\gamma}{2(c_{k} - c_{k+1})} \;.
\end{align*}
Because $\vextradual_{k-1}$ is chosen so that {$a_k (\vdual_{k-1} - \vextradual_{k-1}) = -a_{k-1}(\vdual_{k-1} - \vdual_{k-2})$}, it holds that: 
\begin{align*}
   \frac{a_k}N \sum_{i=1}^{N} &  (\dual_k^i - \extradual_{k-1}^i)\vx_i \cdot (\vw_k - \vu_{\star})
  =\frac{a_{k-1}}N \sum_{i=1}^{N}  (\dual_{k-1}^i - \dual_{k-2}^i)\vx_i \cdot (\vw_{k-1} - \vw_k)\\
  & + \frac{a_k}N \sum_{i=1}^{N}  (\dual_k^i - \dual_{k-1}^i)\vx_i \cdot (\vw_k - \vu_{\star})  - \frac{a_{k-1}}N \sum_{i=1}^{N}  (\dual_{k-1}^i - \dual_{k-2}^i)\vx_i \cdot (\vw_{k-1} - \vv_{\star}),
\end{align*}
where the last two terms telescope.

Summing from $k = 1 $ to $T$, by canceling the telescoping terms, we have that
\begin{align*}
  & \sum_{k=1}^T a_k (L(\vw_k, \vv_{\star}) - L(\vu_{\star}, \vdual_k)) \numhere\label{eq:gap-upper-telescope}\\
  \le & \;   - \frac\gamma2  \sum_{k=1}^{T}\norm{\vdual_{k-1} - \vdual_k}_2^2 + \frac\gamma2 \norm{\vdual_{0} - \vv_{\star}}_2^2 - \frac\gamma2 \norm{\vdual_T - \vv_{\star}}_2^2+ \sum_{k=1}^{T}\frac{a_k^2 \Delta^2\gamma}{2(c_{k} - c_{k+1})}\\
  &\; +  \frac{c_1}{2\gamma}\norm{\vu_{\star} - \vw_{0}}_2^2
   -  \sum_{k=1}^{T} \frac{c_k}{2\gamma}\norm{\vw_k - \vw_{k-1}}_2^2 - \frac{c_{T+1}}{2\gamma}\norm{\vw_T - \vu_{\star}}_2^2 \\
  & \; + \frac{a_T}N \sum_{i=1}^{N}  (y_T^i - \dual_{T-1}^i)\vx_i \cdot (\vw_T - \vu_{\star})  + \sum_{k=1}^{T}\frac{a_{k-1}}N \sum_{i=1}^{N}  (\dual_{k-1}^i - \dual_{k-2}^i)\vx_i \cdot (\vw_{k-1} - \vw_k) \;.
\end{align*}
Focusing on the last line, define $\mB = [\vx_1, \vx_2, \dots, \vx_N]$ and $\mSigma_B = \frac1N \mB \mB^{\top}$. Then we have 
\begin{align*}
  & \; \frac{a_T}N \sum_{i=1}^{N}  (y_T^i - \dual_{T-1}^i)\vx_i \cdot (\vw_T - \vu_{\star})  + \sum_{k=1}^{T}\frac{a_{k-1}}N \sum_{i=1}^{N}  (\dual_{k-1}^i - \dual_{k-2}^i)\vx_i \cdot (\vw_{k-1} - \vw_k)
  \\
  \le & \; \frac{a_T}N \opnorm{\mB} \norm{\vdual_T - \vdual_{T-1}}_2 \norm{\vw_T - \vu_{\star}}_2  + \sum_{k=1}^{T}\frac{a_{k-1}}N \opnorm{\mB} \norm{\vdual_{k-1} - \vdual_{k-1}}_2  \norm{\vw_{k-1} - \vw_k}_2
  \\
  \le & \; \frac{c_{T+1}}{2\gamma} \norm{\vw_T - \vu_{\star}}_2^2 + \sum_{k=2}^{T}\frac{c_k}{2\gamma} \norm{\vw_k - \vw_{k-1}}_2^2 + \sum_{k=1}^{T}\frac{a_k^2\gamma\opnorm{\mSigma_B} }{2 N c_{k+1}} \norm{\vdual_k - \vdual_{k-1}}_2^2,
\end{align*}
where the first step follows from Cauchy-Schwarz and the definition of operator norm, and the second step follows from Young's inequality and $\opnorm{\mSigma_B} = \frac1N \opnorm{\mB}^2$.

Plugging back in \Cref{eq:gap-upper-telescope} and using $c_1 = 2, c_k - c_{k+1} = 1/T$, we have
\begin{align*}
  & \sum_{k=1}^T a_k (L(\vw_k, \vv_{\star}) - L(\vu_{\star}, \vdual_k)) \le - \sum_{k=1}^{T}\Big(\frac12 - \frac{\opnorm{\mSigma_B} a_k^2}{2 N c_{k+1}} \Big)\gamma\norm{\vdual_{k-1} - \vdual_k}_2^2  \\
  & \qquad  + \frac\gamma2 \norm{\vdual_{0} - \vv_{\star}}_2^2 - \frac\gamma2 \norm{\vdual_T - \vv_{\star}}_2^2+ \sum_{k=1}^{T}\frac{T a_k^2 \Delta^2 \gamma}{2} +  \frac1\gamma\norm{\vu_{\star} - \vw_{0}}_2^2  - \frac{1}{2\gamma}\norm{\vw_T - \vu_{\star}}_2^2 \;.
\end{align*}
To bound the $\norm{\vdual_k - \vdual_{k-1}}_2^2$ term from above by 0, we require $ N c_{k+1} \ge a_k^2\opnorm{\mSigma_B}$. Choosing $a_k = (\min_k c_k) \sqrt{N / \opnorm{\mSigma_B}} = \sqrt{N / \opnorm{\mSigma_B}}$ suffices. Then $A_k = k \sqrt{N / \opnorm{\mSigma_B}}$. 

Let $\hat\vw_T = \frac1{A_T} \sum_{k=1}^N a_k\vw_k$ and $\hat\vdual_T = \frac1{A_T} \sum_{k=1}^N a_k\vdual_k$. Jensen's inequality implies that 
\begin{align*}
  & L(\hat\vw_T, \vv_\star) - L(\vu_\star, \hat\vdual_T) \le \frac1{A_T}\sum_{k=1}^T a_k (L(\vw_k, \vv_{\star}) - L(\vv_{\star}, \vdual_k)) 
  \\ 
  \le &\; \frac1{A_T}\Bigl(\frac\gamma2 \norm{\vdual_{0} - \vv_{\star}}_2^2 + \frac\gamma2 T^2 \frac{N}{\opnorm{\mSigma_\mB}} \Delta^2 +  \frac1\gamma\norm{\vu_{\star} - \vw_{0}}_2^2 - \frac{1}{2\gamma}\norm{\vw_T - \vu_{\star}}_2^2 \Bigr)  
  \\
  \le &\; \frac{\sqrt{\opnorm{\mSigma_\mB}}}{T \sqrt{N}}\Bigl(2\zeta^2\gamma N + \frac\gamma2 T^2 \frac{N}{\opnorm{\mSigma_\mB}} \Delta^2 +  \frac1\gamma\norm{\vu_{\star} - \vw_{0}}_2^2 - \frac{1}{2\gamma}\norm{\vw_T - \vu_{\star}}_2^2 \Bigr),\numhere \label{eq:jensen-bound-all-t}
\end{align*}
where the last inequality uses $\zeta$-boundedness of the domain of $g$ using \Cref{fact:lipschitz-conjugate}.

Setting $T = 2 \zeta \sqrt{\opnorm{\mSigma_B}}/ \Delta$ gives 
\begin{equation}\label{eq:stepsize-error-bound}
  L(\hat\vw_T, \vv_\star) - L(\vu_\star, \hat\vdual_T) \le 2\Delta \zeta \gamma \sqrt{N}+ \frac{\norm{\vw_0 - \vu_\star}_2^2 \Delta}{2 \zeta \gamma\sqrt{N}} - \frac{\norm{\vw_T - \vu_\star}_2^2 \Delta}{4 \zeta \gamma\sqrt{N}} . 
\end{equation}

Choosing $\gamma = \norm{\vu_\star - \vw_0} / (\zeta\sqrt N)$ further gives 
\begin{align}\label{eq:jensen-bound}
    L(\hat\vw_T, \vv_\star) - L(\vu_\star, \hat\vdual_T) \le 3 \norm{\vu_\star - \vw_0}_2 \Delta - \frac{\Delta \norm{\vw_T - \vu_\star}_2^2}{4\norm{\vu_\star - \vw_0}} \le 3 \norm{\vu_\star - \vw_0}_2 \Delta.
\end{align}

Recall from \Cref{eq:Fenchel} that $ \frac1N \sum_{i=1}^{N} \ell_i(\vx_i \cdot \vw) + \psi(\vw) = \max_{\vdual \in \R^N} L(\vw, \valpha)$ for all $\vw \in \R^d$.
Setting $\vv_{\star} = \argmax_{\vdual \in \R^N} L(\hat\vw_T, \vdual)$,
we have \(L(\hat\vw_T, \vv_\star) = \frac1N \sum_{i=1}^{N} \ell_i(\vx_i \cdot \hat\vw_T) + \psi(\hat\vw_T)\).
Setting $\vu_{\star} = \vw^*$, we have $L(\vu_\star, \hat\vdual_T) \le \frac1N \sum_{i=1}^{N} \ell_i(\vx_i \cdot \vw^*) + \psi(\vw^*)$. Hence, 
\begin{eqnarray*}
\frac1N \sum_{i=1}^{N} \ell_i(\vx_i \cdot \hat\vw_T) + \psi(\hat\vw_T) = L(\hat\vw_T, \vv_\star) &\le& 3 \norm{\vu_\star - \vw_0}_2 \Delta + L(\vu_\star, \hat\vdual_T) \\ 
&\le& 3 \norm{\vu_\star - \vw_0}_2 \Delta 
+ \frac1N \sum_{i=1}^{N} \ell_i(\vx_i \cdot \vw^*) + \psi(\vw^*) \;,
\end{eqnarray*}
which completes the proof. 
\end{proof}
Note that \Cref{eq:jensen-bound-all-t} directly implies boundedness of iterates:
\begin{corollary}\label{cor:bounded-iterates}
  Under the same setup as in \Cref{prop:main-opt},  for every iteration $k \le T$, it holds that 
  \(\norm{\vw_{k} - \vw^*}_2 \le 4 \norm{\vw_0 - \vw^*}_2. \)
  Therefore, we have $\norm{\vw_k}_2 - \norm{\vw^*}_2 \le 4 \norm{\vw_0}_2 + 4 \norm{\vw^*}_2$ and thus 
  \(\norm{\vw_k}_2 \le 4 \norm{\vw_0}_2 + 5\norm{\vw^*}_2.\)
  In particular, the norm of the convex combination also satisfies
  \begin{equation}\label{eq:wt-bound}
  \norm{\hat\vw_k}_2 = \norm[\Big]{\frac1{A_T} \sum_{k=1}^N a_k\vw_k}_2 = O(\norm{\vw_0}_2 + \norm{\vw^*}_2). 
  \end{equation} 
\end{corollary}
\begin{proof}
  \Cref{eq:jensen-bound-all-t} holds for all $k \in \mathbb{Z}_+$, in particular for all $k \le T = 2 \zeta \sqrt{\opnorm{\mSigma_B}}/ \Delta$. Fix $k \le T$.

  Denote $D_0 = \norm{\vu_\star - \vw_0}_2$. Plugging $\gamma = D_0 / (\zeta\sqrt N)$ into \Cref{eq:jensen-bound-all-t} gives 
  \begin{align*}
    L(\hat\vw_{k}, \vv_\star) - L(\vu_\star, \hat\vdual_{k}) \le \frac{\sqrt{\opnorm{\mSigma_\mB}}}{k \sqrt{N}}
    \Bigl(3\zeta D_0 \sqrt{N} + \frac{D_0 {k}^2 \sqrt{N}\Delta^2 }{2 \zeta \opnorm{\mSigma_\mB}}  - \frac{\zeta\sqrt{N}}{2D_0}\norm{\vw_{k} - \vu_{\star}}_2^2 \Bigr) \;.
  \end{align*}
  Setting $\vu_\star = \vw^*$ and $\vv_{\star} = \argmax_{\vdual \in \R^N} L(\hat\vw_{k}, \vdual)$, it holds that $L(\hat\vw_{k}, \vv_\star) - L(\vu_\star, \hat\vdual_{k}) \ge 0$, because it is the difference of the function value of the problem \eqref{eq:optimization_formulation} at $\hat\vw_{k}$ and the optimal solution $\vw^*$. Thus,
  \(3\zeta D_0 \sqrt{N} + \frac{D_0 {k}^2 \sqrt{N}\Delta^2 }{2 \zeta \opnorm{\mSigma_\mB}}  - \frac{\zeta\sqrt{N}}{2D_0}\norm{\vw_{k} - \vu_{\star}}_2^2 \ge 0,\)
  which simplifies to 
  \[\norm{\vw_{k} - \vu_{\star}}_2^2 
  \le 6 D_0^2 + \frac{D_0^2 {k}^2 \Delta^2 }{ \zeta^2 \opnorm{\mSigma_\mB}}.  \]
  Since $k \le 2 \zeta \sqrt{\opnorm{\mSigma_B}}/ \Delta$ and recall $\vu_\star = \vw^*$ by our choice, we have $\norm{\vw_{k} - \vw^*}_2^2 
  \le 10 D_0^2 = 10 \norm{\vw_0 - \vw^*}_2^2$.
\end{proof}
\new{We highlight that the error in \Cref{prop:main-opt} has dependence on the initial distance $\norm{\vw_0 - \vw^*}_2$ rather than any notion of weight domain radius. This error is optimal up to a constant; see \Cref{thm:lower_bound}.}
\begin{theorem}\label{thm:main-empirical}
Suppose the loss function satisfies \Cref{assump:convex-lipschitz}. A number of $N$ samples are generated per \Cref{assump:dist-assumption} and then $\epsilon$-corrupted per \Cref{def:corruption}. With probability at least $1-\tau$, \Cref{alg:main-opt} outputs $\hat\vw_T$ solving \eqref{eq:w1-dro} with error $O(\sigma \zeta \sqrt\epsilon (\norm{\vw^*}_2 + \norm{\vw_0}_2))$ in the objective function for some sample size $N = O((d \log d + \log (1/\tau))/\epsilon)$. 
\end{theorem}

\begin{proof}
  By \Cref{assump:dist-assumption} and \Cref{lemma:prepend-one-covariance}, the uncorrupted version of covariates $S = \{\vx_i\}_{i=1}^N$ is generated from a distribution of covariance $\Sigma$ with $\opnorm{\Sigma} \le \sigma^2$. By \Cref{fact:bounded_variance_stability_iid}, with probability at least $1-\tau$, the set $S' = \{\vx \in S: \norm{\vx - \vmu_S}_2 \le 2 \sigma \sqrt{d / \epsilon}\}$ satisfies \((\epsilon, O(\sqrt\epsilon))\)-stability and \(|S'| \ge (1 - \epsilon) |S|\). Let $T$ be the given $\epsilon$-corrupted version of $S$. Because \(|S'| \ge 1 - \epsilon\), we can view $T$ as a $2\epsilon$-corrupted version of $S'$. 

  The key idea is that although the algorithm is run on the corrupted samples, its analysis tracks quantities related to the stable samples.   
  More precisely, we establish a coupling argument that all iterates in \Cref{alg:main-opt} (on input corrupted samples $T$ and corresponding labels) are the same as the corresponding iterates in \Cref{alg:main-opt-analysis} (with input samples replaced by stable covariates $S'$ and corresponding labels), where $\vz_k$ in \Cref{alg:main-opt-analysis} is chosen to be $\bar\vz_k$. 
  For $\bar\vz_k$ to be a valid choice, we will verify that $\norm{\bar\vz_k - \frac1N \sum_{i=1}^{N} \extradual_{k-1}^i \vx_i^s}_2 \le \Delta$ for all $k$.

  Let $\cI_{\text{bad}}$ be the indices of the samples in $T \setminus S'$, which include corrupted samples and samples with large covariates.
  We proceed with a proof by induction. For iteration $ k = 1$, both algorithms are initialized at the same point $\vw_0$. Since $\vextradual_0 = \bar\vextradual_0$, \Cref{prop:inexact-hg-oracle} verifies that 
  $\norm{\bar\vz_k - \frac1N \sum_{i=1}^{N} \extradual_{k-1}^i \vx_i^s}_2 \le \Delta$ for $k =1$. Because only samples of indices in $\cI_{\text{bad}}$ are different in the input of both algorithms, $\bar\vdual_1$ is equal to $\vdual_1$ except for indices in $\cI_{\text{bad}}$.

  We inductively assume that both algorithms produce the same iterates at iteration $k-1$, and $\bar\vdual_{k-1}$ is equal to $\vdual_{k-1}$ except for indices in $\cI_{\text{bad}}$. 
  The key observation here is that quantities $\bar\vdual$, $\bar\vextradual$, and $\vx^c$ at the start of iteration $k$ of \Cref{alg:main-opt,alg:main-opt-analysis} only differ in indices $\cI_{\text{bad}}$. Moreover, both $\extradual_k^i$ and $\bar\extradual_k^i$ are $3\zeta$-uniformly bounded for all $i$ and $k$.
  Therefore, both $\{\extradual^i_{k-1} \vx^c_i\}_{i\in[N]}$ and $\{\bar\extradual^i_{k-1} \vx^c_i\}_{i\in[N]}$ are $2\epsilon$-corrupted versions of the stable set $\{\extradual^i_{k-1} \vx^s_i\}$ with the same stability parameters.
  Then \Cref{prop:inexact-hg-oracle} verifies again that 
  $\norm{\bar\vz_k - \frac1N \sum_{i=1}^{N} \extradual_{k-1}^i \vx_i^s}_2 \le \Delta$ for iteration $k$, which implies that both algorithms produce the same iterate $\vw_k$, completing the induction.

  It remains to analyze the convergence guarantee of \Cref{alg:main-opt-analysis}. \Cref{fact:stability_bounded_covariance} and \Cref{fact:bounded_variance_stability_iid} immediately imply that $\opnorm{\mSigma_\mB} = O(\sigma^2)$, where $\mSigma_\mB$ is the empirical covariance matrix of $S'$. Applying \Cref{prop:main-opt} with $\Delta = O(\sigma \zeta \sqrt\epsilon)$ implies that for stable set $\{\vx_i^s\}_{i=1}^N$,
  \(\frac1N \sum_{i=1}^{N} l_i(\vx_i^s \cdot \hat\vw_T) + \psi(\hat\vw_T) \le \frac1N \sum_{i=1}^{N} l_i(\vx_i^s \cdot \vw^*) + \psi(\vw^*) + 3 \norm{\vw^* - \vw_0}_2 \Delta\). We conclude the proof by noting that for all $\vw \in \R^d$, \(\abs{\frac1N \sum_{i=1}^{N} l_i(\vx_i^s \cdot \vw) - \frac1N \sum_{i=1}^{N} l_i(\vx_i \cdot \vw)}= O(\zeta \sigma \norm{\vw}_2\sqrt\epsilon)\).
\end{proof}

\new{By standard uniform convergence results and initializing at $\vw_0 = \vzero$, we get the following corollary whose proof can be found in \Cref{sec:population-loss}.
\begin{restatable}{corollary}{populationmain}
\label{cor:population-error}
Suppose the  loss function satisfies \Cref{assump:convex-lipschitz} and $N$ samples are generated according to \Cref{assump:dist-assumption} and then $\epsilon$-corrupted. Let $\tau \in (0,1)$ and $c((\vx, \thelabel), (\vx', \thelabel')) = \norm{\vx - \vx'}_r + \chi(\thelabel - \thelabel')$. If in addition we assume that $\norm{\vx}_2\le C\sigma\sqrt{d}$ for some constant $C > 0$,
then there exists a sample size $N = O( W^2 d (\log(d) + \log(1/\tau))/\epsilon)$ such that with probability at least $1-\tau$, \Cref{alg:main-opt} outputs $\vw_T$ that satisfies 
\[ \max_{\sQ: W_1^c(\sP_0, \sQ) \le \rho} \E_{\sQ}[\ell(\vw_T; \vxi)] \le \inf_{\vw: \norm{\vw}_2 \le W} \max_{\sQ: W_1^c(\sP_0, \sQ) \le \rho} \E_{\sQ}[\ell(\vw; \vxi)] + O(\sigma \zeta \sqrt\epsilon  \norm{\vw^*}_2).\]
\end{restatable}}

\begin{remark}\label{remark:W2-LR-extension}
    \new{
We remark that our algorithmic framework can be adapted to solve the empirical version of \eqref{eq:w2-linear-regression} in \Cref{example:W2-robust-LR}, which takes the form $\min_\vw \left( \sqrt{\frac{1}{N} \sum_{i=1}^N (\vx_i \cdot \vw - y_i)^2} + \psi(\vw) \right)$, where $\psi(\vw) = \rho \norm{\vw}_s$. Instead of using Fenchel conjugacy on separable loss terms, we can employ the self-duality of $\ell_2$ norm. Let $\vv(\vw) \in \R^N$ be the vector of residuals with entries $v_i(\vw) = \vx_i \cdot \vw - y_i$. The loss term can be written as $\frac{1}{\sqrt{N}}\norm{\vv(\vw)}_2 = \frac{1}{\sqrt{N}} \sup_{\vz \in \R^N: \norm{\vz}_2 \le 1} \vv(\vw)^\top \vz$. This yields the min-max problem $\min_{\vw} \max_{\vz \in \R^N: \norm{\vz}_2 \le 1} \left\{ \frac{1}{\sqrt{N}} \sum_{i=1}^N z_i (\vx_i \cdot \vw - y_i) + \psi(\vw) \right\}$. This structure is amenable to \Cref{alg:main-opt} with two modifications. First, the primal update (Line~\ref{line:primal-update}) requires a robust estimate of the vector $\frac{1}{\sqrt{N}} \sum_{i=1}^N \bar{z}_{k-1}^i \vx_i$. Since the constraint $\norm{\vz}_2 \le 1$ implies that each weight $z_i$ is bounded by $|z_i| \le 1$, this weighted sum can be computed by our robust oracle per \Cref{prop:inexact-hg-oracle}. Second, the dual update (Line~\ref{line:dual-update}) is no longer separable but becomes a joint proximal update for the vector $\vz$, which corresponds to an efficient projection onto the Euclidean unit ball. }
\end{remark}

\new{Finally, we argue that the error in \Cref{thm:main-empirical} is optimal (up to constant factors). We rely on the following lower bound of robust mean estimation for distributions with bounded variance, which can be found in e.g., \citet[Lemma 1.11]{diakonikolas2022algorithmic}.

\begin{fact}[Lower Bound for Robust Mean Estimation]\label{fact:lower_bound_instance}
Let $\sD$ be the family of all one-dimensional distributions with variance at most $\sigma^2$. Let $\epsilon \in (0,1/2)$. Any algorithm with access to $\epsilon$-corrupted samples from an unknown distribution $\sP \in \sD$ incurs an additive error $\Omega(\sigma\sqrt{\epsilon})$ in the worst case to estimate $\E[\sP]$.
\end{fact}

The lower bound is established by constructing two ``hard'' distributions, $\sP_A$ and $\sP_B$, which an adversary can make indistinguishable. Specifically, let $\sP_A(0)= 1-\epsilon$ and $\sP_A(\sigma/\sqrt{\epsilon}) = \epsilon$ and $\sP_B(0) = 1-\epsilon$ and $\sP_B(-\sigma/\sqrt{\epsilon}) = \epsilon$. It is easy to verify that $\Var(\sP_A) = \Var(\sP_B) = \sigma^2(1-\epsilon) < \sigma^2$ and     
\begin{align*}
    d_{\mathrm{TV}}(\sP_A, \sP_B) &= \frac{1}{2} \sum_{x \in \{-\sigma/\sqrt{\epsilon},\; 0,\; \sigma/\sqrt{\epsilon}\}} |\sP_A(x) - \sP_B(x)|\\
    &= \frac{1}{2} ( |0 - \epsilon| + |(1-\epsilon) - (1-\epsilon)| + |\epsilon - 0| ) = \epsilon.
\end{align*}
By \citet[Proposition 1.7]{diakonikolas2022algorithmic}, if $d_{\mathrm{TV}}(\sP_A, \sP_B) \le 2\epsilon$, an adversary can make them indistinguishable under $\epsilon$-contamination by outputting contaminated distribution $(\sP_A + \sP_B)/2$.

Now we establish the lower bound for the optimality gap based on \Cref{fact:lower_bound_instance}.
\begin{lemma}[Lower Bound on Optimality Gap]\label{thm:lower_bound}
Let $\epsilon \in (0,1/2)$ and $D>0$. For any $\zeta>0$ and $\sigma>0$, there exists a $\zeta$-Lipschitz generalized linear loss $\ell(\vw;\vx)$ and two one-dimensional distributions, $\sP_A$ and $\sP_B$, both with variance at most $\sigma^2$ (as in \Cref{fact:lower_bound_instance}), such that any algorithm given $\epsilon$-corrupted samples from an unknown distribution $\sP \in \{\sP_A, \sP_B\}$ and returning $\hat{\vw}$ for the problem $\min_{\vw\in[-D/2, D/2]}\E_{\sP}[\ell(\vw;\vx)]$ must incur an optimality gap
\[
\max_{\sP \in \{\sP_A, \sP_B\}} \left\{ \E_{\sP}[\ell(\hat{\vw};\vx)]-\min_{\vw\in[-D/2, D/2]}\E_{\sP}[\ell(\vw;\vx)] \right\} = \Omega\big(\zeta D\sigma\sqrt{\epsilon}\big).
\]
\end{lemma}

\begin{proof}
The proof is a reduction from the one-dimensional robust mean estimation lower bound. We adopt the hard instance distributions $\sP_A$ and $\sP_B$ from \Cref{fact:lower_bound_instance}.

Consider the one-dimensional data space $\vx\in\R$ and the parameter interval $\vw\in[-D/2, D/2]$, which has diameter $D$. We choose the linear loss function $\ell(\vw;\vx) = -\zeta\vw \cdot \vx$. This loss is convex in $\vw$ and $\zeta$-Lipschitz in $\vw \cdot \vx$. The population objective for a distribution $\sP$ with mean $\mu$ is
\[
\E_{\vx\sim\sP}[\ell(\vw;\vx)] = -\zeta\vw\E_{\sP}[\vx] = -\zeta\mu\vw.
\]
The unique minimizer of this objective over $\vw \in [-D/2, D/2]$ is $\vw^* = (D/2) \cdot \operatorname{sign}(\mu)$.

We now consider the two hard distributions $\sP_A$ and $\sP_B$ from \Cref{fact:lower_bound_instance}, with means $\mu_A = +\sigma\sqrt{\epsilon}$ and $\mu_B = -\sigma\sqrt{\epsilon}$ and the corresponding optimum $\vw^*_A = +D/2$ and $\vw^*_B = -D/2$.
    
By \Cref{fact:lower_bound_instance}, an adversary can provide an $\epsilon$-corrupted dataset to the algorithm that is statistically indistinguishable from a corrupted sample of $\sP_A$ or $\sP_B$. Let the algorithm's output on this corrupted dataset be $\hat{\vw} \in [-D/2, D/2]$.

We now show that $\hat{\vw}$ has a large optimality gap for at least one of these two distributions. We compute the gap for both cases:
\begin{align*}
    \Gap_A(\hat{\vw}) &:= \E_{\sP_A}[\ell(\hat{\vw};\vx)] - \E_{\sP_A}[\ell(\vw^*_A;\vx)] = (-\zeta \mu_A \hat{\vw}) - (-\zeta \mu_A (D/2)) = \zeta \mu_A (D/2 - \hat{\vw}).
\end{align*}
Similarly, we have
\begin{align*}
    \Gap_B(\hat{\vw}) &:= \E_{\sP_B}[\ell(\hat{\vw};\vx)] - \E_{\sP_B}[\ell(\vw^*_B;\vx)] = (-\zeta \mu_B \hat{\vw}) - (-\zeta \mu_B (-D/2)) \\
    &= \zeta (-\mu_B) (\hat{\vw} + D/2) = \zeta \mu_A (\hat{\vw} + D/2) \qquad (\text{since } -\mu_B = \mu_A).
\end{align*}
    
The worst-case gap for the algorithm is $\max(\Gap_A(\hat{\vw}), \Gap_B(\hat{\vw}))$, whose minimum is achieved when $\Gap_A(\hat{\vw}) = \Gap_B(\hat{\vw})$, which requires
\(
\zeta \mu_A (D/2 - \hat{\vw}) = \zeta \mu_A (\hat{\vw} + D/2) \iff \hat{\vw} = 0
\).

The minimum of $\max(\Gap_A(\hat{\vw}), \Gap_B(\hat{\vw}))$ is thus
\(
\Gap_A(0) = \zeta \mu_A (D/2 - 0) = \zeta (c\sigma\sqrt{\epsilon}) (D/2) = \Omega(\zeta D\sigma\sqrt{\epsilon}).
\)
Since any output $\hat{\vw}$ must have a gap of at least this much in one of the two indistinguishable cases, the lower bound holds for any algorithm.
\end{proof}
}
\subsection{Application to SVMs: Removing Dependence on the Optimal Solution Norm}
\new{The error bound in \Cref{cor:population-error} depends on the norm of the optimal solution, $\norm{\vw^*}_2$, which may be large or unknown. We now show that this dependence can be removed for certain structured problems by leveraging stronger distributional assumptions beyond a bounded covariance.

We present a key application to Support Vector Machine (SVM) classification with the hinge loss. By assuming the covariate distribution satisfies a standard anti-concentration property, we can provably restrict the search space for the model weights. This allows us to establish an error rate $O(\epsilon^{1/4})$, independent of $\norm{\vw^*}_2$. This subsection, following a similar line of arguments to \citet[Section 3]{diakonikolas2019sever}, provides the necessary assumptions and detailed analysis for this result.}

We consider a binary classification problem. Let the samples be $\{(\vx_i, y_i)\}_{i=1}^N \subset \R^{d} \times \{\pm 1\}$, drawn from a reference distribution $\sP_0$. The transportation cost for the Wasserstein distance is given by $c((\vx, y), (\vx', y')) = \chi(y - y') + \norm{\vx - \vx'}_r$, where $\chi(0)=0$ and $\chi(z)=\infty$ for $z \ne 0$ (ensuring labels do not change during transport), and $\norm{\cdot}_r$ is the $\ell_r$-norm for covariates. The loss function is the hinge loss, $\ell(\vw; \vx, y) = (1 - y \vw \cdot \vx)_+$.

The $W_1$-DRO problem is 
\[\min_{\vw \in \R^d} \max_{\sP: W_1^c(\sP, \sP_0) \le \rho } \E_{(\vx,y) \sim \sP} [(1 - y \vw \cdot \vx)_+] \;.\]
As established in \Cref{sec:dro-reg} (specifically, \Cref{eq:w1-dro-regularization}), and noting that the hinge loss $(1-z)_+$ is $1$-Lipschitz (so $\zeta=1$), this problem is equivalent to:
\[\min_{\vw \in \R^d} \E_{(\vx,y) \sim \sP_0} [(1 - y \vw \cdot \vx)_+] + \rho \norm{\vw}_s, \]
where $1/s + 1/r = 1$. In the remaining subsection, we consider $r=2$, which implies $s=2$ (Euclidean norms).

\begin{assumption}[Distributional Assumptions on Covariates]\label{assump:svm-x-marginal-appendix}
We assume the following for the marginal distribution of covariates $\vx$ under $\sP_0$, similar to~\citet{diakonikolas2019sever}:
\begin{itemize}[nosep,leftmargin=*]
  \item \emph{Bounded Second Moment:} $\E_{\sP_0}[\vx \vx^T] \preceq \sigma^2 \mI_d$, where $\mI_d$ is the $d \times d$ identity matrix. This implies $\E_{\sP_0}[(\vv \cdot \vx)^2] \le \sigma^2$ for any unit vector $\vv \in \R^d$.
  \item \emph{Anti-concentration:} There exists a constant $B > 0$ and $\lambda > 0$ such that for all unit vectors $\vv \in \R^d$, it holds that $\Pr_{(\vx,y) \sim \sP_0}[|\vv \cdot \vx| \le \lambda \sigma] \le B \lambda$.
\end{itemize}
\end{assumption}

With \Cref{assump:svm-x-marginal-appendix}, we can show that restricting the weight vector $\vw$ to an $\ell_2$-ball of radius $1/(\lambda \sigma)$ incurs at most an $O(\lambda)$ additive error in the expected hinge loss. This is formalized in the following lemma.
\begin{lemma}[Clipping of $\vw$ for Hinge Loss]\label{lemma:svm_clipping_w}
  Let $\vw_P = \Proj_{\cB(1/(\lambda \sigma))}(\vw) = \min\{1, 1/(\lambda \sigma\norm{\vw}_2)\} \vw$ be the projection of $\vw$ onto the $\ell_2$-ball of radius $1/(\lambda \sigma)$. Then,
  \[ \E_{(\vx,y) \sim \sP_0} [(1 - y (\vw_P \cdot \vx))_+] \le \E_{(\vx,y) \sim \sP_0} [(1 - y (\vw \cdot \vx))_+] + 2 B \lambda. \]
\end{lemma}
\begin{proof}
  This proof adapts the argument from~\citet[Lemma E.9]{diakonikolas2019sever} to the hinge loss.
  Let $\ell(\hat{\vw}; \vx, y) = (1 - y (\hat{\vw} \cdot \vx))_+$. Fix $\vw \in \R^d$. If $\norm{\vw}_2 \le 1/(\lambda \sigma)$, then $\vw_P = \vw$, and the statement is trivial.
  Assume $\norm{\vw}_2 > 1/(\lambda \sigma)$, so $\norm{\vw_P}_2 = 1/(\lambda \sigma)$. Also, $\vw_P = k\vw$ for $k = 1/(\lambda \sigma \norm{\vw}_2) < 1$.

  Define the event $E = \{(\vx, y) : | \vx \cdot \vw_P | > 1 \}$.
  On $E$: If $y (\vw_P \cdot \vx) > 1$, then $\ell(\vw_P; \vx, y) = 0$. Since $y (\vw \cdot \vx) = (1/k) y (\vw_P \cdot \vx) > (1/k) \cdot 1 \ge 1$, it follows that $\ell(\vw; \vx, y) = 0$. Thus, $\ell(\vw_P; \vx, y) = \ell(\vw; \vx, y)$.
  If $y (\vw_P \cdot \vx) < -1$, then $\ell(\vw_P; \vx, y) = 1 - y (\vw_P \cdot \vx)$. Since $k<1$, $y (\vw_P \cdot \vx) = k y (\vw \cdot \vx)$. Since $y (\vw \cdot \vx) < 0$, $y (\vw_P \cdot \vx) \ge y (\vw \cdot \vx)$. So $1 - y (\vw_P \cdot \vx) \le 1 - y (\vw \cdot \vx)$. As both are positive, $\ell(\vw_P; \vx, y) \le \ell(\vw; \vx, y)$.
  Thus, on $E$, we have $\ell(\vw_P; \vx, y) \le \ell(\vw; \vx, y)$.

  On the complement event $E^c = \{(\vx, y) : |\vx \cdot \vw_P| \le 1 \}$:
  We have $\ell(\vw_P; \vx, y) = (1 - y (\vw_P \cdot \vx))_+$. Since $y \in \{\pm 1\}$ and $|\vx \cdot \vw_P| \le 1$, $y (\vw_P \cdot \vx)$ is between $-1$ and $1$.
  So, $1 - y (\vw_P \cdot \vx)$ is between $0$ and $2$. Thus, $\ell(\vw_P; \vx, y) \le 2$.
  As $\ell(\vw; \vx, y) \ge 0$, on $E^c$ we have $\ell(\vw_P; \vx, y) \le 2 \le 2 + \ell(\vw; \vx, y)$.

  We now bound $\sP_0(E^c)$. The condition $|\vx \cdot \vw_P| \le 1$ implies $| \vx \cdot (\vw_P / \norm{\vw_P}_2) | \le 1 / \norm{\vw_P}_2 = \lambda \sigma$.
  Let $\vv = \vw_P / \norm{\vw_P}_2$. Then $\norm{\vv}_2 = 1$.
  So, $\sP_0(E^c) = \sP_0(|\vx \cdot \vw_P| \le 1) = \sP_0(|\vx \cdot \vv | \le \lambda \sigma)$. By the anti-concentration assumption (\Cref{assump:svm-x-marginal-appendix}), $\sP_0(E^c) \le B \lambda$.

  Therefore, by the law of total expectation:
  \begin{align*}
    \E[\ell(\vw_P; \vx, y)] &= \E[\ell(\vw_P; \vx, y) | E] \sP_0(E) + \E[\ell(\vw_P; \vx, y) | E^c] \sP_0(E^c) \\
    &\le \E[\ell(\vw; \vx, y) | E] \sP_0(E) + \E[2 + \ell(\vw; \vx, y) | E^c] \sP_0(E^c) \\
&= \E[\ell(\vw; \vx, y)] + 2\sP_0(E^c) \le \E[\ell(\vw; \vx, y)] + 2 B \lambda.
  \end{align*}
\end{proof}
This lemma allows us to restrict the search space for $\vw$ to the ball $\cB(1/(\lambda\sigma))$ at the cost of an additive $O(\lambda)$ term in the objective. The next fact provides a uniform convergence guarantee over this bounded set after removing a small fraction of data points with large norms.

\begin{fact}[Uniform Convergence for SVM, adapted from {\citet[Lemma C.3]{diakonikolas2019sever}}]\label{fact:svm_uniform_convergence}
Let the set $\{(\vx_i, y_i)\}_{i=1}^N$ be $N$ i.i.d.\ samples from a distribution $\sP_0$ satisfying \Cref{assump:svm-x-marginal-appendix}. Let $S = \{ (\vx_i, y_i) : \norm{\vx_i}_2 \le 80 \sigma \sqrt{d / \epsilon}\}$. 
Then for a sample size $N = O(d \log(d /  \lambda) / \epsilon)$, with probability at least $0.9$, it holds that $|S| \ge (1-\epsilon) N $ and for all $\vw \in \cB((\lambda \sigma)^{-1})$,
\[ \left|\frac{1}{|S|}\sum_{i \in S} \ell(\vw; \vx_i, y_i) - \E_{\sP_0}[\ell(\vw; \vx, y)]\right| = O((1 + \lambda^{-1})\sqrt{\epsilon}). \]
\end{fact}

To obtain the $O(\epsilon^{1/4})$ error bound, we set the parameter $\lambda$ from the anti-concentration assumption appropriately. The error from clipping $\vw$ (Lemma \ref{lemma:svm_clipping_w}) is $O(B\lambda)$. The uniform convergence error (Fact \ref{fact:svm_uniform_convergence}), if $\lambda \ll 1$, is $O(\lambda^{-1}\sqrt{\epsilon})$.
Our main algorithm is applied to the set $S$, which can be seen as being $2 \epsilon$-corrupted if the original data had $\epsilon$-corruptions and we discard an additional $\epsilon$ fraction based on norm. The error from our main algorithm (\Cref{prop:main-opt}) is $O(\sigma \zeta \sqrt{2\epsilon} W) = O(\sigma \sqrt{\epsilon} (\lambda\sigma)^{-1}) = O(\lambda^{-1}\sqrt{\epsilon})$ since $\zeta=1$ for hinge loss. Total error from these sources is $O(B\lambda + \lambda^{-1}\sqrt{\epsilon})$.

We want to balance these errors. 
To minimize this quantity, we can set $\lambda = \lambda^{-1}\sqrt{\epsilon}$, which implies $\lambda^2 = \sqrt{\epsilon}$, so $\lambda = \epsilon^{1/4}$.
Plugging this $\lambda$ back, the error becomes $O(B \epsilon^{1/4})$.
This shows that an overall error of $O(\epsilon^{1/4})$ in the objective function value is achievable viewing $B$ as an absolute constant. The radius of the weight ball is $W = 1/(\lambda\sigma) = 1/(\sigma\epsilon^{1/4})$.

\subsection{Tuning the Step Size Parameter}

\new{A practical issue with \Cref{alg:main-opt} is that its step size parameter $\gamma$ depends on the unknown quantity $\norm{\vw_0 - \vw^*}_2$ (denoted by $D_0$ hereafter). 
Fortunately, this is not a significant barrier. We can tune $\gamma$ using a standard geometric search for this distance, which finds a sufficiently good estimate by running the algorithm a logarithmic number of times. This section details this procedure, including a robust method for evaluating candidate solutions and an optional early stopping criterion. The overhead for this tuning is a logarithmic factor in the runtime and a small additive term in the final error bound.}

\subsubsection{A Bounded Geometric Search with Early Stopping}

We assume access to a loose upper bound $W_0$ on $D_0 = \norm{\vw_0 - \vw^*}_2$. This assumption provides a concrete upper bound on the noise level across the entire search, ensuring our stopping condition remains statistically meaningful. 

We search for an estimate of $D_0$ starting from a plausible minimum $D_{\min} = \Delta/\zeta$ and generate a sequence of candidates $D_j = D_{\min} \cdot 2^j$ for $j = 0, 1, 2, \dots, J_{\max}$, where $J_{\max} = \lceil \log_2(W_0/D_{\min}) \rceil$ ensures the search space covers the full range up to $W_0$.

For each $j=0, 1, \dots, J_{\max}$:
\begin{enumerate}[nosep,leftmargin=*]
  \item Run \Cref{alg:main-opt} with $\gamma_j = D_j / (\zeta\sqrt{N})$ to obtain a solution $\hat{\vw}_T^{(j)}$.
  \item Robustly estimate the true (population) objective value $F_{\sP_0}(\hat{\vw}_T^{(j)})$ using the available $\epsilon$-corrupted training data, which is discussed next. Let this noisy estimate be $\hat{F}^{(j)}$.
  \item Check for early stopping \eqref{eq:early-stopping}. If the condition is met, terminate the search.
\end{enumerate}

After the search terminates, select the solution $\hat{\vw}_T^{(k)}$ that corresponds to the minimum estimated value, $\min_j \hat{F}^{(j)}$, found among all runs performed.
\subsubsection{Robust Estimation of Objective Function Value}
For each candidate solution $\hat{\vw}_T^{(j)}$, we need to estimate the objective function value \[F_{\sP_0}(\hat{\vw}_T^{(j)}) = \E_{(\vx,y) \sim \sP_0} [\ell(\hat{\vw}_T^{(j)}; \vx, y)] + \psi(\hat{\vw}_T^{(j)}),\] where $\psi(\vw) = \rho \zeta \norm{\vw}_s$. The regularization term $\psi(\hat{\vw}_T^{(j)})$ can be computed directly. The challenge is to estimate the expected loss $\E_{(\vx,y) \sim \sP_0} [\ell(\hat{\vw}_T^{(j)}; \vx, y)]$ using the $\epsilon$-corrupted training samples $\{(\vx_i^c, y_i^c)\}_{i=1}^N$.

We compute the loss values $L_i^{(j)} = \ell(\hat{\vw}_T^{(j)}; \vx_i^c, y_i^c)$ for $i=1, \dots, N$. This set $\{L_i^{(j)}\}_{i=1}^N$ constitutes an $\epsilon$-corrupted sample of the random variable $L^{(j)}(\vxi) = \ell(\hat{\vw}_T^{(j)}; \vxi)$, where $\vxi \sim \sP_0$. We can robustly estimate the mean of $L^{(j)}(\vxi)$ from these corrupted samples. Standard robust mean estimation techniques for 1-dimensional data (such as those based on truncated mean~\cite[Section 1.4.2]{diakonikolas2022algorithmic} or variants of \Cref{alg:robust_mean_estimation} adapted to 1-dimensional settings) achieve an estimation error of $O(\sigma_{L^{(j)}} \sqrt{\epsilon})$; see \Cref{fact:robust_mean_estimation_with_stability}, where $\sigma_{L^{(j)}}$ is the standard deviation of the clean loss values $L^{(j)}(\vxi)$.

To bound $\sigma_{L^{(j)}}$, observe that the loss function $\ell(\vw; \vx, y) = \phi(\vw \cdot \vx, y)$ has $\phi(\cdot,y)$ as $\zeta$-Lipschitz. The iterates $\hat{\vw}_T^{(j)}$ are bounded. By \Cref{cor:bounded-iterates}, $\norm{\hat{\vw}_T^{(j)}}_2 = O(\norm{\vw_0}_2 + \norm{\vw^*}_2)$. By \Cref{assump:dist-assumption}, the arguments $\vw \cdot \vx$ will have variance $\norm{\vw}_2^2 \sigma^2 $. The Lipschitz property of the loss $\ell_y$ implies that the standard deviation $\sigma_{L^{(j)}}$ of the loss values is $O(\zeta \norm{\hat{\vw}_T^{(j)}}_2 \sigma)$. Consequently, the error in estimating \(\E_{(\vx,y) \sim \sP_0} [\ell(\hat{\vw}_T^{(j)}; \vx, y)]\) is $E_{\text{eval}} = O(\zeta (\norm{\vw_0}_2 + \norm{\vw^*}_2) \sigma \sqrt{\epsilon})$.

\subsubsection{Early Stopping Condition}
At each step $k$, we keep track of the best (lowest) estimated value found so far, $\hat{F}_{\text{best}} = \min_{j<k} \hat{F}^{(j)}$. We terminate the search if the current estimate is significantly worse:
\begin{equation}\label{eq:early-stopping}
  \text{Stop at step } k \text{ if } \hat{F}^{(k)} > \hat{F}_{\text{best}} + 3 E_{\text{ub}} \;,
\end{equation}
where $E_{\text{ub}}$ is the upper bound of $E_{\text{eval}}$ that we know how to compute: $O(\zeta (\norm{\vw_0}_2 + W_0) \sigma \sqrt{\epsilon})$. 

This stopping criterion has the following probabilistic guarantee: If the error bound $|\hat{F}^{(j)} - F^{(j)}| \le E_{\text{eval}}$ holds for all estimates (which occurs with high probability and the dependence on the failure probability is logarithmic), an observed increase of more than $2E_{\text{eval}}$ guarantees that the true objective value has also increased. Specifically, if the condition triggers at step $k$, we have $F_{\sP_0}^{(k)} > F_{\sP_0}^{(j_{\text{best}})}$ with high probability, where $j_{\text{best}}$ is the index of the previous best run. This ensures we do not stop prematurely due to noise. 

\subsubsection{Overall Error Guarantee and Complexity}
The geometric search ensures that for at least one candidate $D_j$, we have $D_0 \le D_j < 2D_0$. For this $D_j$, running \Cref{alg:main-opt} with $\gamma_j = D_j / (\zeta\sqrt{N})$ yields a solution $\hat{\vw}_T^{(j)}$. As shown in the analysis of the error term in \Cref{prop:main-opt} when $\gamma$ is mismatched (see \Cref{eq:stepsize-error-bound}), the function value suboptimality $O(D_0 \Delta)$ becomes $O((D_j + D_0^2/D_j)\Delta)$. Since $D_j \in [D_0, 2D_0]$, this is again $O(D_0\Delta) = O(\norm{\vw_0 - \vw^*}_2 \sigma \zeta \sqrt{\epsilon})$.

If the true $D_0$ is even smaller than $\Delta/\zeta$, our search does not try $D_j$ values below $\Delta/\zeta$.  For the first run $D_{\min} = \Delta/\zeta$, the proven error bound (see \Cref{eq:stepsize-error-bound}) depends on $(D_{\min} + D_0^2/D_{\min})\Delta$ and would be approximately $(\Delta/\zeta + D_0^2\zeta/\Delta)\Delta = \Delta^2/\zeta + D_0^2\zeta$. Since $D_0 < \Delta/\zeta$, $D_0^2\zeta < (\Delta^2/\zeta^2)\zeta = \Delta^2/\zeta$. So the error in this case is $O(\Delta^2/\zeta) = O(\sigma^2 \zeta \epsilon)$. 

The final excess risk bound for the selected solution $\hat{\vw}_T^{(*)}$ remains 
$O(D_0 \Delta + E_{\text{eval}}) = O(\zeta (\norm{\vw_0}_2 + \norm{\vw^*}_2) \sigma \sqrt{\epsilon} + \sigma^2 \zeta \epsilon)$.
Initializing at $\vw_0 = \vzero$, this simplifies to $O(\zeta (\norm{\vw^*}_2 + \sigma \sqrt \epsilon) \sigma \sqrt{\epsilon})$.

The number of times \Cref{alg:main-opt} is executed is at most $O(\log(W_0\zeta/\Delta))$, adding this logarithmic factor to the total computational complexity. This procedure makes the algorithm adaptive to the unknown distance $\norm{\vw_0 - \vw^*}_2$, incurring total error $O(\zeta \sigma \sqrt{\epsilon} (\norm{\vw^*}_2 + \sigma \sqrt\epsilon))$.

\section{Additional Related Work}
While \Cref{sec:dro-modeling} provided a detailed critique of the most immediate competing models, this section reviews the broader context for our work. We cover relevant background on the tractability and generalization properties of DRO, as well as related methods in robust stochastic optimization.

\subsection{The Role of DRO Against Overfitting}
Beyond its role in safeguarding against distributional shifts between training and testing environments, Distributionally Robust Optimization (DRO) is also recognized for its capacity to mitigate overfitting and enhance model generalization. Specifically, DRO can also be used to ensure estimators from an empirical distribution $\eP(N)$ generalize well to the true data-generating distribution $\sP_0$ (i.e., assuming $\pptest = \sP_0$). A notable theoretical advantage is that generalization error bounds for DRO estimators can often be derived without explicit reliance on traditional measures of hypothesis class complexity, such as VC dimension or Rademacher complexity~\cite[Section 4]{shafieezadeh2019regularization}. This contrasts with standard Empirical Risk Minimization (ERM), where such complexity measures are typically central to the generalization analysis~\cite[Chapters 4, 6, 26]{shalev2014understanding}. Consequently, DRO's excess risk bounds are sometimes preferred due to these minimal assumptions or observed improvements in overfitting~\cite[Example 2]{kuhn2019wasserstein}. By optimizing for worst-case performance within an ambiguity set centered on $\eP(N)$ (and chosen to contain $\sP_0$), DRO inherently encourages solutions less sensitive to training data idiosyncrasies, promoting better out-of-sample performance (i.e., performance on testing distributions).

While this ability to control overfitting and provide strong generalization guarantees is an important feature of DRO, it is crucial to distinguish this role from its application to handling unknown, fixed distributional shifts---which is the main concern of the present work. 
The practical tuning of the ambiguity set radius, $\rho$, often reflects this distinction. When the primary goal is to prevent overfitting and ensure convergence to the true underlying distribution $\sP_0$ as more data becomes available, practitioners typically advocate for a radius $\rho$ that shrinks with the sample size $N$ (e.g., $\rho \to 0$ as $N \to +\infty$). For instance, specific rates like $\rho(N) = \tilde\Theta(N^{-1/2})$ for Wasserstein ambiguity sets have been proposed to optimize this convergence while avoiding issues like the curse of dimensionality~\cite{chen2018robust, gao2023finite}. This ensures that the DRO model eventually learns the data-generating distribution accurately, assuming it is well-specified within the model class. Conversely, when DRO is employed to ensure robustness against potential, persistent distributional shifts—whose magnitude is assumed to be independent of the sample size---the radius $\rho$ is usually chosen as a constant. This constant reflects the anticipated level of mismatch between the training and deployment environments and is not diminished with more training data, as the shift itself is not expected to vanish~\cite{blanchet2024distributionally}. Our paper focuses on the latter scenario (constant $\rho$ for distributional shifts), compounded by the challenge of training data contamination, rather than on the specific mechanisms by which DRO mitigates statistical overfitting through vanishing ambiguity.

The endeavor to achieve good generalization to $\sP_0 = \pptest$ faces additional hurdles when the training data $\eP(N)$ itself stems from a contaminated source $\ppcorrupt$, rather than $\sP_0$. Some alternative frameworks, such as particular interpretations of Outlier-Robust Wasserstein DRO (OR-WDRO), might attempt to address this by ensuring $\sP_0$ lies within an ambiguity set defined relative to $\ppcorrupt$ (e.g., $W_{p,\epsilon}(\ppcorrupt, \sP_0 ) \le \rho$), while still aiming to evaluate performance on $\sP_0$. However, as further explored in \Cref{sec:related-models}, applying such models to handle \emph{training set contamination} within a DRO generalization context can prove insufficient. These approaches might not fully capitalize on the assumed structure or properties of $\sP_0$ (if it is also the clean target distribution) and can yield guarantees with problematic aspects, such as dependence on data dimensionality. This highlights the necessity for a more direct and principled method for dealing with training data contamination---even if the end goal is classical generalization---further motivating the distinct modeling strategy for contamination that our work adopts.

\subsection{Tractability of Wasserstein DRO}
In this subsection, we survey key formulations and algorithmic approaches related to the tractability of Wasserstein DRO. This review provides context and evidence supporting the reasonableness of the structural assumptions on loss and cost functions adopted in our work, which enable the DRO problem to be reformulated as a tractable regularized risk minimization problem.

The tractability of solving the Wasserstein DRO problem, $\min_{\vw} \sup_{\sQ: W_{p}(\sQ, \eP(N)) \le \rho}\E_{\vxi \sim \sQ}[\ell(\vw, \vxi)]$, depends critically on the interplay between the loss function $\ell(\vw, \vxi)$, the order $p$ of the Wasserstein distance, and the ground cost $c(\cdot, \cdot)$ used to define $W_p$. For a fixed weight vector $\vw \in \R^d$, the inner supremum, which represents the worst-case expectation, has several important characterizations. Key formulations, as discussed in works like \citet{gao2023distributionally} and \citet[Corollary 2]{gao2024wasserstein}, include:
\begin{align*}
R_S(\vw) &:= \sup_{W_{p}(\sQ, \eP(N)) \le \rho}\E_{\vxi \sim \sQ}[\ell(\vw, \vxi)], \\
R_I(\vw) &:= \inf_{\lambda > 0} \Bigl\{ \lambda \rho^{p} + \E_{\vxi_0 \sim \eP(N)}\Bigl[\sup_{\vxi'} \left(\ell(\vw; \vxi') - \lambda c^{p}(\vxi_0, \vxi')\right)\Bigr] \Bigr\},  \\
R_U(\vw) &:= \Bigl((\E_{\vxi \sim \eP(N)}[\ell(\vw; \vxi)])^{1/p} + \rho  L_{\eP(N), \ell}(\vw) \Bigr)^{p}.
\end{align*}
The formulation \(R_I\) is a dual representation derived from Kantorovich duality, and the equality $R_S(\vw) = R_I(\vw)$ holds under very mild conditions. The formulation \(R_U\) seeks to express the worst-case risk as a direct modification or regularization of the empirical risk, where $L_{\eP(N), \ell, c}(\vw)$ depends on the loss, samples, cost, and potentially $\vw$ itself~\cite{chu2024wasserstein}, and the equality $R_S(\vw) = R_U(\vw)$ requires somewhat strong conditions\footnote{Early results for formulation \(R_U\) were only developed for Lipschitz generalized linear models with cost functions $c((\vx, \thelabel), (\vx', \thelabel'))$ equal to either $\norm{\vx - \vx'}_r + \chi(\thelabel - \thelabel')$ or $\norm{\vx - \vx'}_r + \abs{\thelabel - \thelabel'}$ for $r \in [1, +\infty]$~\cite{shafieezadeh2019regularization}. More recent results allow for broader combinations of loss functions under   various assumptions and more general cost functions; see~\citet[Table 1]{chu2024wasserstein} for a summary. \citet{chu2024wasserstein} is the first work to the best of our knowledge that allows for nonsmooth loss functions (either convex or nonconvex) that are broader than piecewise linear losses, but explicitness and computability of $L_{\eP(N), \ell, c}(\vw)$ remains a strong condition.}.

While general forms of \(R_U\) might involve complex and data-dependent generalized Lipschitz terms $L_{\eP(N), \ell, c}(\vw)$ that are often not explicit, under even stronger conditions on the problem, $\R_U$ admits the following straightforward formulation: $\min_{\vw} (\E_{\vxi \sim \eP(N)}[\ell(\vw; \vxi)] + \rho \psi(\vw) )$. If $\ell$ is convex in $\vw$ and the regularization term $\psi$ is also convex, the overall DRO problem is generally considered tractable and can often be solved efficiently using first-order methods, even if nonsmooth, provided that subroutines like proximal operators are readily computable. 

The dual formulation \(R_I\), on the other hand, leads to solving a nested optimization problem. Tractability here often relies on a specific problem structure. For instance, if $\ell(\vw; \vxi)$ is convex in $\vw$ and concave in $\vxi$ (or a maximum of such functions), the problem can be framed as a convex-concave saddle-point problem~\cite[Remark 9]{gao2023distributionally}. However, the requirement for $\ell(\vw; \vxi)$ to be concave in $\vxi$ is very restrictive and is not satisfied by most common machine learning losses (e.g., logistic or hinge loss), limiting the applicability of such direct minimax solution techniques.

Various specialized algorithms have been developed to address the min-max structure suggested by $R_I(\vw)$ even without global concavity in $\vxi$. For example, \citet{blanchet2022optimal} design algorithms for problems with twice differentiable loss functions $\ell$ that are locally strongly convex in $\vw$. Other research has explored specific loss functions. For instance, \citet{li2019first} studied logistic regression and \citet{li2020fast}  addressed $\ell_2$-regularized SVMs (both potentially under more general cost functions $c(\cdot, \cdot)$ than studied in this paper, therefore the equivalence to a straightforward $R_U$ reformulation is lost). This means that a unified algorithmic framework based on \(R_I\) that matches the breadth of loss functions (convex, Lipschitz, but not necessarily smooth or concave in $\vxi$) leading to the simpler \(R_U\) formulations considered in this paper is not readily available.

Considering these different avenues for tractability, Wasserstein DRO problems that can be reformulated as regularized risk minimization problems constitute a large and critically important class of solvable DRO models. This class notably includes $W_1$-DRO with general convex and Lipschitz loss functions combined with cost functions that yield a norm regularization on $\vw$, a setting extensively analyzed by \citet{shafieezadeh2019regularization} and central to our  work. The fact that these problems can often be solved using well-understood (though potentially nonsmooth) convex optimization techniques makes them particularly relevant for practical applications.

\subsection{Robust Stochastic Optimization}
An avenue for handling nonsmooth objective functions in optimization is through smoothing techniques, prominently featuring the Moreau envelope. Given a nonsmooth function $F(\vw)$, its Moreau envelope is defined as $F_\lambda(\vw) = \min_{\vu} \{F(\vu) + \frac{1}{2\lambda}\|\vu-\vw\|^2_2\}$, which yields a smooth approximation to $F(\vw)$. The work of \citet{jambulapati2021robust-regressi} explores robustly optimizing such smoothed objectives. However, this methodology presents significant challenges when considering its direct applicability to obtaining guarantees for the original (unregularized and nonsmooth) function $F(\vw)$.

First, the theoretical guarantees provided in \citet[Section 5]{jambulapati2021robust-regressi} are for a further modified objective, specifically $F_\lambda(\vw) + (\mu/2)\|\vw\|^2_2$, where an additional quadratic regularization (with strength $\mu/2 > 0$) is introduced on top of the Moreau smoothing. This combined objective, $F_\lambda(\vw) + (\mu/2)\|\vw\|^2_2$, benefits from $(1/\lambda)$-smoothness (inherited from $F_\lambda$) and, if $F$ is convex, $\mu$-strong convexity. These properties of smoothness and strong convexity are central to their convergence analysis. Translating these convergence guarantees back to the original nonsmooth and unregularized function $F(\vw)$ appears nontrivial and was not provided in \citet{jambulapati2021robust-regressi}. 

Second, certain assumptions within their framework appear to sidestep the core challenge of non-differentiability. Specifically, Assumption 3 in \citet{jambulapati2021robust-regressi}, which is crucial for their Algorithm 6 (ApproxMoreauMinimizer), posits access to a stochastic first-order oracle that returns ``gradients'' of the original, potentially nonsmooth, expected loss $F^*(\vw) = \E_{\sP_0}[\ell(\vw; \vx, y)]$. If the function $F^*$ is genuinely nonsmooth, its gradient will not exist at all points $\vw$. Assuming the existence and accessibility of such gradients everywhere, especially around local optima\footnote{Although nonsmooth, mild conditions such as local Lipschitzness ensure that a function is almost everywhere differentiable. So by perturbing a little, one may guarantee access to the gradient with probability 1. However, if the function is nondifferentiable at optima, perturbation can give potentially large gradients that carry no information how close to a solution the algorithm may be, so in this sense nonsmoothness remains a concern.}, effectively diminishes the primary motivation for employing the Moreau envelope, which is designed specifically to create a smooth surrogate precisely where the original function lacks differentiability. This suggests that their proposed algorithm and its analysis might be implicitly tailored to nonsmooth functions of a specific and relatively benign structure where this assumption is not overly restrictive, rather than general nonsmooth convex functions. 
Given these considerations, \citet{jambulapati2021robust-regressi} does not directly offer a solution for robustly optimizing the broad class of unregularized and nonsmooth generalized linear models considered in our work.

\section{Conclusion}

This paper initiates a formal study of distributionally robust optimization with outliers in the training data, with the focus on rigorous error guarantees and polynomial-time algorithms. It provides the first polynomial-time algorithm guaranteeing error $O(\sqrt{\epsilon})$ under mild distributional and loss function assumptions. We hope that this work will stimulate further research in this area, leading to guarantees for other ambiguity sets and families of possibly nonconvex loss functions.

\vspace{2em}

\appendix
\crefalias{section}{appendix}
\crefalias{subsection}{appendix}

\section{A Basic Algorithm for Robust Mean Estimation\label{sec:appendix-prelim}}

We state here a basic deterministic robust mean estimation 
algorithm that satisfies error guarantee stated in \Cref{fact:robust_mean_estimation_with_stability}. We 
highlight that state-of-the-art algorithms for robust mean 
estimation can be implemented in near-linear time, requiring 
only a logarithmic number of passes over the data; see, 
e.g.,~\citet{ChengDG19, dong2019quantum, diakonikolas22streaming}. 

\begin{algorithm}[htp]
  \caption{\(\mathsf{RobustMeanEstimation(T, \eps)}\) }
  \label{alg:robust_mean_estimation}
  \KwIn{\(0 < \epsilon < 1/2\) and \(T \subset \R^d \) is an \(\epsilon\)-corrupted set}
  \KwOut{\(\hat \vmu\) with \(\norm{\hat\vmu - \vmu_{S}} = \bigO[\big]{\sqrt{\opnorm{\mSigma} \epsilon}}\)}
  Initialize a weight function \(q : T \rightarrow \R_{+}\) with \(q(\vz)= 1/ |T|\) for all \(\vz \in T\)\\
  \While{\(\sum_{\vz \in T}{q(\vz)} \ge 1 - 2\epsilon\)}{
    \(\hat\vmu \leftarrow \sum_{\vz\in T}q(\vz) \vz / \sum_{\vz \in T}{q(\vz)}, \quad \hat\mSigma \leftarrow \sum_{\vz\in T}q(\vz)(\vz - \hat\vmu)(\vz-\hat\vmu)^{\top} / {\sum_{\vz \in T}{q(\vz)} }\)\\
    Compute the top eigenvector \(\vv\) of \(\hat\mSigma\) and define
\(h(\vz) := |\vv^{\top}(\vz - \hat\vmu)|^{2}\)\\
    Find the largest threshold \(t > 0\) such that \(\sum_{\vz\in T: h(\vz) \ge t}q(\vz) \ge \epsilon\)\\
    \(f(\vz) := h(\vz)\Ind{\{h(\vz) \ge t\}}, \quad q(\vz) \leftarrow q(\vz)\left( 1 - \frac{f(\vz)}{\max_{\vz' \in T: q(\vz') \neq 0} f(\vz')}\right) \quad (\forall \vz \in T)\) 
  }
  \Return\(\hat\vmu\)
\end{algorithm}

\section{Auxiliary Results for Section \ref{sec:dro-modeling}\label{sec:dro-modeling-appendix}}
\new{This appendix provides supplementary material for \Cref{sec:dro-modeling}. In \Cref{sec:doro-appendix}, we present a detailed counterexample illustrating the failure of the heuristic algorithm proposed for the DORO framework, substantiating the critique in \Cref{sec:related-models}. The subsequent sections contain missing proofs and other technical details supporting \Cref{sec:dro-modeling}.}

\subsection{A Counterexample for the DORO Algorithm\label{sec:doro-appendix}}

The main text introduces the Distributionally Robust Outlier-aware Optimization (DORO) framework by \citet{zhai2021doro} as an approach that, like ours, considers both training data contamination (modeled by $\TV(\sP_0, \ppcorrupt) \le \eps$) and distributional shifts in test data. DORO aims to achieve this by optimizing an $f$-divergence based DRO risk over an optimally chosen $(1-\epsilon)$-subpopulation of the training data. \new{This appendix elaborates on the limitations of DORO's proposed computational algorithm, as highlighted in the main body, by providing a specific counterexample where it fails.}

From a computational standpoint, \citet[Algorithm 1]{zhai2021doro} propose an algorithm for minimizing the \emph{DORO risk}, \new{defined as the minimum $f$-divergence-based DRO risk (with shift parameter $D_f(\pptest, \ppcorrupt) \le \rho$) among all possible $(1-\epsilon)$-subpopulations of the training data.} However, this algorithm is presented without formal convergence guarantees or a computational complexity analysis, leaving its efficacy for provably approximating the DORO risk minimizer uncertain. 

We present a specialization of their method to the $\CVaR^\alpha$ setting in Algorithm~\ref{alg:doro-algorithm-cvar-1}. 
\begin{algorithm}[ht]
    \KwIn{Batch size $N_{batch}$; outlier fraction hyperparameter $\epsilon \in [0,1)$}
    \For{each training iteration}{
        Sample a batch $\{\vxi_i\}_{i=1}^{N_{batch}}$ from the (contaminated) training distribution $\ppcorrupt$\\
        Compute losses: $\ell_i \leftarrow \ell(\vw; \vxi_i)$ for $i=1, \dots, N_{batch}$\\
        Sort the losses: let $\ell_{(1)} \ge \ell_{(2)} \ge \dots \ge \ell_{(N_{batch})}$ be the sorted losses, with $\vxi_{(j)}$ being the sample corresponding to $\ell_{(j)}$\\
        Let $N_{kept} \leftarrow N_{batch} - \lfloor \epsilon N_{batch} \rfloor$ \tcp{Number of samples kept}
        Define the objective for finding $\eta$:
        $$G(\eta; \vw) =  \alpha^{-1}\frac{1}{N_{kept}} \sum_{j=\lfloor \epsilon N_{batch} \rfloor+1}^{N_{batch}} (\ell_{(j)}-\eta)_+  +\eta$$
        Find $\eta^* \leftarrow \argmin_{\eta \in \mathbb{R}} G(\eta; \vw)$\label{line:dual-optimal}\\
        Update $\vw$ by taking one step of a gradient-based method to minimize $G(\eta^*; \vw)$\\
    }
    \caption{DORO for $\CVaR^\alpha$ (adapted from {\citet[Algorithm 1]{zhai2021doro}})}\label{alg:doro-algorithm-cvar-1}    
\end{algorithm}
More generally, \citet[Algorithm 1]{zhai2021doro} attempts to minimize DORO-risk by iteratively:
\begin{enumerate}[nosep]
    \item Sorting samples of batch size $n$ based on current losses $\ell_j(\vw) = \ell(\vw; \vxi_j)$.
    \item Discarding the $\epsilon n$ samples with the largest losses, denoting the remaining samples as $S_{kept}$.
    \item Finding $\eta^*$ that minimizes some dual form $G(\vw, \eta)$ of the DRO risk evaluated on the empirical distribution of the current $S_{kept}$.
    \item Updating $\vw$ using $G(\vw, \eta^*)$.
\end{enumerate}

Let us consider the specific instantiation of robust mean estimation with $\ell(\vw; \vxi) = \norm{\vw - \vxi}_2^2$ using $\CVaR^\alpha$ with  $\alpha = 1$, so that the CVaR operator $\CVaR^1$ becomes the expectation $\E$; see \Cref{eq:cvar-definition} and note that $\argmin_\vw \E_{\sP}[\norm{\vw - \vxi}_2^2] = \E_\sP[\vxi]$ (the best point estimate for a distribution under a squared-error loss function is the mean).

To find $\eta^*$ that minimizes this $ G(\vw,\eta) = \left( \frac{1}{N_{kept}} \sum_{j \in S_{kept}} (\ell_j(\vw)-\eta)_+ \right) +\eta$ as in Line~\ref{line:dual-optimal}, consider an individual term $h(\eta; \ell) = (\ell-\eta)_+ + \eta$.
\begin{itemize}[nosep]
    \item If $\eta < \ell$, then $(\ell-\eta)_+ = \ell-\eta$, so $h(\eta; \ell) = \ell-\eta+\eta = \ell$.
    \item If $\eta \ge \ell$, then $(\ell-\eta)_+ = 0$, so $h(\eta; \ell) = \eta$.
\end{itemize}
Thus, $\inf_{\eta} h(\eta; \ell) = \ell$, achieved for any $\eta \le \ell$.
For the sum $F(\vw,\eta) = \E_{S_{kept}} [(\ell(\vw;\vxi) - \eta)_+] + \eta$, the optimal $\eta^*$ can be chosen as any value less than or equal to $\min_{j \in S_{kept}} \ell_j(\vw)$. For such an $\eta^*$, we have $(\ell_j(\vw)-\eta^*)_+ = \ell_j(\vw)-\eta^*$.
Hence,
$$ F(\vw, \eta^*) = \frac{1}{N_{kept}} \sum_{j \in S_{kept}} (\ell_j(\vw)-\eta^*) +\eta^* = \frac{1}{N_{kept}} \sum_{j \in S_{kept}} \ell_j(\vw). $$
Therefore, in this specialized setting (CVaR with $\alpha=1$), DORO's Algorithm 1 simplifies to minimizing the average loss over the $(1-\epsilon)$ fraction of training samples that currently exhibit the \emph{smallest} losses. This is a form of trimmed-loss estimation.

While trimming samples with large losses seems intuitively robust, the iterative nature of \Cref{alg:doro-algorithm-cvar-1} can be problematic. Consider the following counterexample:
Let the clean data $\vxi$ be drawn i.i.d.\ from the standard normal distribution ${N}(\vzero, \mI_d)$. Suppose an $\epsilon$-fraction of the data consists of outliers fixed at the point $\sqrt{d}\mathbf{e}_1$, where $\mathbf{e}_1$ is the first standard basis vector; then a naive filter that only looks at the norm of the data will not identify those outliers, because $\norm{\vxi}_2 \approx \sqrt{d}$ for nearly all data $\vxi$~\cite[Section 3.1]{vershynin2018hdp} for large enough $d$ and sample size. Now consider how \Cref{alg:doro-algorithm-cvar-1} progresses:
\begin{enumerate}[leftmargin=*, nosep]
    \item \emph{Initialization}: Assume \Cref{alg:doro-algorithm-cvar-1} initializes $\vw$ at the empirical mean of the contaminated data. With high probability and large enough sample size, $\vw_0 \approx \epsilon\sqrt{d}\mathbf{e}_1$ (biased towards the outliers).
    \item \emph{Loss Calculation \& Sorting}: Losses are $\ell(\vw_0; \vxi_j) = \norm{\vw_0 - \vxi_j}_2^2$. Samples $\vxi_j$ closest to $\vw_0$ will have the smallest losses. The outliers at $\sqrt{d}\mathbf{e}_1$ are relatively close to $\vw_0$ compared to clean samples $\vxi \sim {N}(\vzero, \mI_d)$ that happen to be in the direction opposite to $\mathbf{e}_1$ (e.g., near $-\sqrt{d}\mathbf{e}_1$).
    \item \emph{Filtering}: Algorithm 1 discards the $\epsilon n$ samples with the highest losses. These are likely to be clean samples that are far from the currently biased estimate $\vw_0$. The outliers, being closer to $\vw_0$, are likely retained in $S_{kept}$.
    \item \emph{Update}: The next estimate $\vw_1$ is found by minimizing the average loss over $S_{kept}$. Since $S_{kept}$ disproportionately contains outliers and clean samples near the outliers, $\vw_1$ (the mean of $S_{kept}$ in this quadratic loss case) will likely move even closer to the outliers at $\sqrt{d}\mathbf{e}_1$.
\end{enumerate}
This iterative process may lead to the estimate $\vw$ converging towards the outliers rather than the true mean $\vzero$, as the algorithm repeatedly filters out ``extreme'' clean data points relative to its increasingly biased estimate. This demonstrates that \Cref{alg:doro-algorithm-cvar-1} incurs at least $\sqrt{d} \epsilon$ error in the parameter estimation and thus $\epsilon^2 d$ error in loss value, not achieving the dimension-free error as claimed in their theoretical result.

\new{In sum, while the DORO algorithm is computationally appealing and its framework's motivation aligns with our modeling principles in \Cref{sec:modeling-principles}, it is ultimately a heuristic. As our counterexample for robust mean estimation demonstrates, its reliance on a loss-based filtering strategy can cause the estimator to converge towards outliers rather than the true parameter. This algorithmic shortcoming underscores the challenges in designing provably robust methods for this setting and motivates the need for a principled approach with formal guarantees, as developed in our work.}

\subsection{Proof of Fact \ref{fact:wasserstein-loss-moment}\label{sec:proof-wasserstein-loss-moment}}
\begin{proof}
 We make use of the following definition in the proof.

\begin{definition}
Let $\cS_1$ and $\cS_2$ be two measurable spaces, and let $T: \cS_1 \to \cS_2$ be a measurable mapping. If $\mu$ is a probability measure on $(\cS_1, \Sigma_1)$, the \emph{pushforward measure} $T_{\#}\mu$ is a probability measure on $\cS_2$ defined for measurable sets $A \subset \cS_2$ by 
\[(T_{\#}\mu)(A) = \mu(T^{-1}(A)) = \mu(\{\vx \in \cS_1 : T(\vx) \in A\})\] Equivalently, for any bounded measurable function $f: \cS_2 \to \R$:\[\int_{\cS_2} f(s_2) d(T_{\#}\mu)(s_2) = \int_{\cS_1} f(T(s_1)) d\mu(s_1) \;.\]
\end{definition}
In the context of this proof, $\cS_1 = \cS_2 = \R^d \times \R$, $\mu = \sP_0$, and $T$ is a sequence of transportation maps $T_n$. The random variable $(\mX', Y')$ drawn from $\sQ_n = (T_n)_{\#}\sP_0$ has the same distribution as $T_n(\mX, Y)$ where $(\mX, Y)$ is drawn from $\sP_0$.

    Let $\mathbf{u}_{\vw} = \vw / \norm{\vw}$. The non-negativity of $\ell_y$ and the limit condition imply that for any $y \in \R$, there exists $M_{y}$ such that $\ell_y(z) \ge (C/2)\abs{z}^{p+\iota}$ for all $\abs{z} > M_{y}$.
    Since $\sP_0^\mX$ is continuous, for any sufficiently small $\epsilon > 0$, we can find a compact set $A \subset \R^d$ such that $\sP_0[\mX \in A] = \epsilon$. Let $K_A = \sup_{\vx \in A} \norm{\vx}$. We can choose $A$ such that $K_A$ is finite (e.g., by choosing $A$ within a fixed ball).

    We proceed to construct a sequence of perturbed measures $\sQ_n$.
    For each $n \in \mathbb{N}$, let $R_n > 0$ be a sequence such that $R_n \to \infty$. Define $\mathbf{z}_n = R_n \mathbf{u}_{\vw}$.
    Choose $\epsilon_n = \frac{\rho^p}{(K_A+R_n)^p + 1}$. For $R_n$ large enough, $\epsilon_n$ can be made arbitrarily small. Since $\sP_0^\mX$ is continuous, we can find a compact set $A_n \subset \R^d$ (e.g., within a fixed bounded region, so $K_{A_n} = \sup_{\vx \in A_n} \norm{\vx}$ is uniformly bounded by some $K_{\max}$) such that $\sP_0[\mX \in A_n] = \epsilon_n$.
    Define a transport map $T_n: \R^d \times \R \to \R^d \times \R$ by:
    $T_n(\vx, y) = (\mathbf{z}_n, y)$ if $\vx \in A_n$,
    $T_n(\vx, y) = (\vx, y)$ if $\vx \notin A_n$.
    Let $\sQ_n = (T_n)_{\#} \sP_0$ be the pushforward measure.

    Next we verify that $W_p(\sP_0, \sQ_n) \le \rho$ for all $n$.
    The Wasserstein-$p$ distance $W_p(\sP_0, \sQ_n)$ is bounded by the cost of the map $T_n$:
    \begin{align*} 
      & \;W_p(\sP_0, \sQ_n)^p \le \E_{\sP_0}[d((\mX,Y), T_n(\mX,Y))^p] \\ 
      = & \; \int_{A_n \times \R} (\norm{\vx - \mathbf{z}_n} + \kappa \abs{y-y})^p d\sP_0(\vx,y) + \int_{(\R^d \setminus A_n) \times \R} (\norm{\vx - \vx} + \kappa\abs{y-y})^p d\sP_0(\vx,y) \\ 
      = & \; \int_{A_n} \norm{\vx - \mathbf{z}_n}^p d\sP_0^\mX(\vx) \le \sP_0[\mX \in A_n] \sup_{\vx \in A_n} \norm{\vx - \mathbf{z}_n}^p \\ 
      \le & \; \epsilon_n (K_A + R_n)^p \quad (\text{since } \norm{\mathbf{z}_n}=R_n \text{ and } \norm{\vx} \le K_A \text{ for } \vx \in A_n) \\ 
      = & \; \frac{\rho^p (K_A+R_n)^p}{(K_A+R_n)^p + 1} \le \rho^p \;.
      \end{align*}
    Thus, $W_p(\sP_0, \sQ_n) \le \rho$ for all $n$.

    We finally verify that the expected loss under $\sQ_n$ goes to $\infty$. We have
    $\E_{\sQ_n}[\ell_{Y'}(\vw \cdot \mX')] = \E_{\sP_0}[\ell_Y(\vw \cdot \mX) \mathbb{I}_{\mX \notin A_n}] + \E_{\sP_0}[\ell_Y(\vw \cdot \mathbf{z}_n) \mathbb{I}_{\mX \in A_n}]$.
    Since $\ell_Y \ge 0$, the first term is non-negative. The second term is 
    \begin{equation}\label{eq:second-term}
      \int_{A_n} \E_{Y|\mX}[\ell_Y(\norm{\vw}R_n)] d\sP_0^\mX(\vx),
    \end{equation}
    where we use the notation $\E_{Y|\mX}$ to denote the conditional expectation over $Y$ given $\mX$.
    By Fatou's lemma, it holds almost surely that 
    $$\liminf_{n\to\infty} \E_{Y|\mX}\left[\frac{\ell_Y(R_n\norm{\vw})}{(R_n\norm{\vw})^{p+\iota}}\right] \ge \E_{Y|\mX}\left[\liminf_{n\to\infty} \frac{\ell_Y(R_n\norm{\vw})}{(R_n\norm{\vw})^{p+\iota}}\right] = C.$$
    Thus, for $n$ sufficiently large, we have $$\E_{Y|\mX}[\ell_Y(R_n\norm{\vw})] \ge \frac{C}{2}(R_n\norm{\vw})^{p+\iota}.$$
    Plugging it back into \Cref{eq:second-term}, we have 
    \begin{align*}
      & \; \E_{\sP_0}[\ell_Y(\vw \cdot \mathbf{z}_n) \mathbb{I}_{\mX \in A_n}] \ge \sP_0[\mX \in A_n] \frac{C}{2} (\norm{\vw}R_n)^{p+\iota} \\
      = & \; \epsilon_n \frac{C}{2} (\norm{\vw}R_n)^{p+\iota} = \frac{\rho^p}{(K_A+R_n)^p + 1} \frac{C}{2} (\norm{\vw}R_n)^{p+\iota}.
    \end{align*}
    As $R_n \to \infty$, this term behaves as $$\frac{\rho^p C \norm{\vw}^{p+\iota}}{2} \frac{R_n^{p+\iota}}{R_n^p} = \frac{\rho^p C \norm{\vw}^{p+\iota}}{2} R_n^{\iota} \to \infty $$ because $\iota > 0$. Thus, $\E_{\sQ_n}[\ell_{Y'}(\vw \cdot \mX')] \to \infty$.
\end{proof}

\subsection{Prepending covariates by 1 preserves bounded covariance assumption\label{sec:prepending-by-one}}

\begin{lemma}[Prepending covariates by \(1\) preserves bounded covariance]\label{lemma:prepend-one-covariance}
	Let \(\{\tilde\vx_i\}_{i=1}^N\) be independent samples from a distribution over \(\sR^{d-1}\) with mean \(\tilde\vmu\) and covariance \(\tilde\mSigma\) satisfying \(\opnorm{\tilde\mSigma} \le \sigma^2\).
	Define \(\vx_i = [1, \tilde\vx_i^\top]^\top \in \sR^d\) by prepending a \(1\) to each \(\tilde\vx_i\).
	Then \(\{\vx_i\}\) can be viewed as i.i.d.\ samples drawn from a distribution with covariance $\mSigma$ satisfying \(\opnorm{\mSigma} \le \sigma^2\).
\end{lemma}
\begin{proof}
Each new vector \(\vx_i \in \sR^d\) can be written as
	\(
	\vx_i = [ 1, \tilde\vx_i^\top ]^\top.
	\) 
	The mean of the new distribution is
	\(
	\vmu = \E[\vx_i] = [ 1 , \tilde\vmu^\top]^\top.
	\) 
	The covariance matrix \(\mSigma\) is defined as \(\mSigma = \E\left[(\vx_i - \vmu)(\vx_i - \vmu)^\top\right]\).

	Expanding \(\vx_i - \vmu\),
	\[
	\vx_i - \vmu = [ 1 - 1, (\tilde\vx_i - \tilde\vmu)^\top ]^\top = [ 0, (\tilde\vx_i - \tilde\vmu)^\top ]^\top. 
	\]
Hence, the covariance matrix becomes
	\[
	\mSigma = \E\left[\begin{bmatrix} 0 \\ \tilde\vx_i - \tilde\vmu \end{bmatrix} \begin{bmatrix} 0 & (\tilde\vx_i - \tilde\vmu)^\top \end{bmatrix}\right] = \begin{bmatrix} 0 & \vzero^\top \\ \vzero & \tilde\mSigma \end{bmatrix},
	\]
	where \(\vzero \in \sR^{d-1}\) is the zero vector.
	
	Thus, \(\mSigma\) is a block diagonal matrix with a zero entry in the top-left and \(\tilde\mSigma\) as the bottom-right block. 
Since the eigenvalues of a block diagonal matrix are the eigenvalues of the diagonal blocks, the eigenvalues of \(\mSigma\) are the eigenvalues of \(\tilde\mSigma\) together with an additional zero eigenvalue. 
Therefore, the operator norm of \(\mSigma\) satisfies
	\(
	\opnorm{\mSigma} = \max(\opnorm{\tilde\mSigma}, 0) = \opnorm{\tilde\mSigma}
	\).
	Given that \(\opnorm{\tilde\mSigma} \le \sigma^2\) by assumption, we conclude
	\(
	\opnorm{\mSigma} \le \sigma^2
	\).
\end{proof}

\subsection{Handling Uncentered Data with Arbitrary Mean\label{sec:arbitrary-mean}}
The main analysis, particularly \Cref{assump:dist-assumption}, assumes that the $(d-1)$-dimensional covariates $\tilde{\vx}$ (before prepending the constant 1) are drawn from a distribution with zero mean ($\E[\tilde{\vx}] = \vzero$). In this subsection, we detail how this assumption can be relaxed to handle covariates with an arbitrary unknown mean $\tilde{\vmu}$, ensuring that our main theoretical guarantees remain applicable. The strategy involves robustly estimating $\tilde{\vmu}$ from the contaminated data, centering the covariates using this estimate, and then showing that the subsequent analysis holds with minor modifications that do not change the overall error rates.

Suppose we are given $N$ samples $\{\tilde{\vx}_i^c, y_i^c\}_{i=1}^N$, where $\tilde{\vx}_i^c \in \R^{d-1}$ are the original covariates from an $\epsilon$-corrupted dataset (superscript $c$ stands for ``corrupted''). The underlying clean covariates $\tilde{\vx}_i$ are drawn i.i.d.\ from a distribution $\sP_0^{\tilde{\vx}}$ with mean $\tilde{\vmu}$ and covariance matrix $\tilde{\mSigma}$ satisfying $\opnorm{\tilde{\mSigma}} \le \sigma^2$.

Our first step is to obtain a robust estimate of $\tilde{\vmu}$. The set of observed covariates $\{\tilde{\vx}_i^c\}_{i=1}^N$ is an $\epsilon$-corrupted version of $N$ i.i.d.\ samples from $\sP_0^{\tilde{\vx}}$. By \Cref{fact:bounded_variance_stability_iid}, with high probability (at least $1-\tau$ for $N = \tilde{O}(d\log(1/\tau)/\epsilon)$), a $(1-\epsilon)$ fraction of the clean samples forms an $(\epsilon, O(\sqrt{\epsilon}))$-stable set, denoted by $\tilde S'$, with respect to $\tilde{\vmu}$ and $\sigma^2$. Consequently, the observed samples $\{\tilde{\vx}_i^c\}_{i=1}^N$ can be treated as a $2\epsilon$-corrupted version of $\tilde S'$. We apply the robust mean estimation algorithm, $\mathsf{RobustMeanEstimation}$ (\Cref{alg:robust_mean_estimation}), to $\{\tilde{\vx}_i^c\}_{i=1}^N$ with a corruption parameter of $2\epsilon$. According to \Cref{fact:robust_mean_estimation_with_stability}, the output $\hat{\vmu}$ satisfies
\[\norm{\hat{\vmu} - \tilde\vmu}_2 = O(\sigma\sqrt{\epsilon}). \] 

Using this estimate $\hat{\vmu}$, we transform the covariates that will be used by our main algorithm:
\begin{enumerate}[nosep]
    \item Center the original $(d-1)$-dimensional covariates: $\bar{\vx}'_i = \tilde{\vx}_i^c - \hat{\vmu}$.
    \item Prepend a constant 1 to form the $d$-dimensional covariates: $\vx_i^c = [1, (\bar{\vx}'_i)^{\top}]^{\top}$.
    \item Run \Cref{alg:main-opt} using these transformed covariates $\{\vx_i^c, y_i^c\}_{i=1}^N$.
\end{enumerate}

We now verify how this transformation affects the premises of our analysis. Consider the stable versions of these transformed covariates, $\vx_i^s = [1, (\tilde{\vx}_i^s - \hat{\vmu})^{\top}]^{\top}$, where $\tilde{\vx}_i^s \in \tilde S'$. Let $S' = \{\vx_i^s\}_{i=1}^N$.

Since $\tilde S'$ is $(\epsilon, O(\sqrt{\epsilon}))$-stable with respect to $\tilde{\vmu}$, by \Cref{fact:stability_bounded_covariance}(i), $\norm{\vmu_{\tilde S'} - \tilde{\vmu}}_2 \le \epsilon$. 
Thus, the mean of $S'$ is 
$\vmu_{S'} = [1, (\vmu_{\tilde S'} - \hat{\vmu})^{\top}]^{\top}$. 
Its squared $\ell_2$-norm is 
\[\norm{\vmu_{S'}}_2^2 = 1^2 + \norm{{\vmu_{\tilde S'}} - \hat{\vmu}}_2^2 \le 1 + 2 \norm{\vmu_{\tilde S'} - \tilde{\vmu}}_2^2 + 2 \norm{\hat{\vmu} - \tilde{\vmu}}_2^2 = 1 + O(\sigma^2\epsilon).\]
By \Cref{lemma:prepend-one-covariance} and because translating a set by a fixed vector does not change the covariance matrix, $S'$ is stable with respect to the mean $\vmu_{S'}$ and variance $\sigma^2$.

By \Cref{lemma:bounded-sequence-scaling-stability}, combining the above mean estimation and stability parameters, we have that for a $3\zeta$-bounded sequence $(\beta_i)$, the weighted sum $\{\beta_i\vx_i^s\}$ is $(\bigO\epsilon, \bigO{\sqrt{\epsilon}})$-stable with respect to its own empirical mean and variance $9\zeta^2(\sigma^2 + \norm{\vmu_{S'}}_2^2)$, which is of order $O(\sigma^2\zeta^2)$.
Using the same argument as in \Cref{prop:inexact-hg-oracle}, we can construct an oracle that satisfies \Cref{assump:approximation-guarantee} via robust mean estimation algorithms with $\Delta = O(\zeta\sigma\sqrt\epsilon)$.

Note that here and elsewhere in the main body, we implicitly assume that $\sigma \ge 1$. If this assumption fails, we can either replace $\sigma$ with $\sqrt{\sigma^2 + 1}$ everywhere, or prepend the covariates by $\sigma$ instead of prepending $1$.
The crucial point is that this small deviation in the mean of the transformed stable covariates does not alter the order of magnitude of $\Delta$.

Finally, we translate these guarantees back to the original problem with uncentered covariates.
For a vector $\vw \in \R^d$, let $w^0$ denote its first coordinate and $\tilde \vw$ denote the rest of coordinates. Define 
\begin{align*}
  J(\vw; \vmu) & = \E_{\sP_0} [\phi(w^0 + \tilde{\vw} \cdot (\tilde{\vx} - {\vmu}), y)] + \rho \zeta \norm{\vw}_s, \\
  J'(\vw; \vmu) & = \frac1{\abs{S'}}\sum_{i: \tilde\vx_i \in \tilde S'} [\phi(w^0 + \tilde{\vw} \cdot (\tilde{\vx_i} - {\vmu}), y_i)] + \rho \zeta \norm{\vw}_s \;.
\end{align*}
The main algorithm finds ${\vw}_T = [{w}^0_T, {\tilde{\vw}}_T^{\top}]^{\top}$\footnote{In the main body, the output is denoted by $\hat\vw_T$. We drop the hat  on $\vw$ here to simplify the notation in this section.} that is an approximate minimizer of $J(\vw; \hat{\vmu})$. 
Specifically, \Cref{cor:population-error} implies that $\vw_T$ is $O((\norm{\vw^*_{\hat{\vmu}}}  + \norm{\vw_0}_2) \Delta)$-suboptimal for this problem, where $\vw^*_{\hat{\vmu}}$ is the true minimizer of $J'(\vw; \hat{\vmu})$.

The original problem we aim to solve corresponds to centering with the true mean $\tilde{\vmu}$: minimizing $J(\vw; \tilde{\vmu})$. Let $\vw^*_{\tilde{\vmu}}$ be its minimizer.
The difference in the objective function value for a given $\vw$ due to using $\hat{\vmu}$ instead of $\tilde{\vmu}$ is:
$|J(\vw; \tilde{\vmu}) - J(\vw; \hat{\vmu})| = |\E_{\sP_0}[\phi(w^0 + \tilde{\vw} \cdot (\tilde{\vx} - \tilde{\vmu}), y) - \phi(w^0 + \tilde{\vw} \cdot (\tilde{\vx} - \hat{\vmu}), y)]|$.
Since $\phi(\cdot, y)$ is $\zeta$-Lipschitz, this difference is bounded by $\E_{\sP_0}[\zeta | (w^0 + \tilde{\vw} \cdot (\tilde{\vx} - \tilde{\vmu})) - (w^0 + \tilde{\vw} \cdot (\tilde{\vx} - \hat{\vmu})) |] = \zeta |\tilde{\vw} \cdot (\hat{\vmu} - \tilde{\vmu})| \le \zeta \norm{\tilde{\vw}}_2 \norm{\hat{\vmu} - \tilde{\vmu}}_2 = \zeta \norm{\tilde{\vw}}_2 O(\sigma\sqrt{\epsilon})$. 
Thus,
\begin{align*} 
  J(\vw_T; \tilde{\vmu}) 
  &  \le J(\vw_T; \hat{\vmu}) + \zeta \norm{\tilde{\vw}_T}_2 O(\sigma\sqrt{\epsilon}) \\ 
  &\le J(\vw^*_{\hat{\vmu}}; \hat{\vmu}) + O((\norm{\vw^*_{\hat{\vmu}}}_2 + \norm{\vw_0}_2) \Delta) + \zeta \norm{{\tilde{\vw}}_T}_2 O(\sigma\sqrt{\epsilon}) \\ 
  &\le J(\vw^*_{\tilde{\vmu}}; \hat{\vmu}) + O((\norm{\vw^*_{\hat{\vmu}}}_2 + \norm{\vw_0}_2) \Delta)  + \zeta \norm{{\tilde{\vw}}_T}_2 O(\sigma\sqrt{\epsilon}) \\ 
  &\le J(\vw^*_{\tilde{\vmu}}; \hat{\vmu}) +O((\norm{\vw^*_{\hat{\vmu}}}_2 + \norm{\vw_0}_2) \Delta) + \zeta (\norm{{\tilde{\vw}}_T}_2 + \norm{\vw^*_{\tilde{\vmu}}}_2) O(\sigma\sqrt{\epsilon}). 
\end{align*}
Assuming the norms of  $\norm{\vw^*_{\hat{\vmu}}}_2$ and $\norm{\tilde\vw^*}_2$ (therefore, also $\norm{{\tilde{\vw}}_T}_2$, as per \Cref{eq:wt-bound}) are both of the same order as $\norm{\vw^*_{\tilde{\vmu}}}_2$ (the norm of the optimal solution to the ideally centered problem), the total excess error remains $O(\zeta\sigma\sqrt{\epsilon} \norm{\vw^*_{\tilde{\vmu}}}_2)$.
The learned parameters $\vw_T = [\hat{w}^0_T, \tilde{\vw}_T^{\top}]^{\top}$ can be mapped to the parameters of a model $w^0_{\text{orig}} + \tilde{\vw}_{\text{orig}} \cdot \tilde{\vx}$ for the original, uncentered covariates by setting $\tilde{\vw}_{\text{orig}} = \tilde{\vw}_T$ and $w^0_{\text{orig}} = \hat{w}^0_T - \tilde{\vw}_T^{\top} \hat{\vmu}$. All guarantees apply to the performance of $\tilde{\vw}_{\text{orig}}$ when evaluated using the ideal centering $w^0 + \tilde{\vw} \cdot (\tilde{\vx} - \tilde{\vmu})$.
Therefore, the overall guarantees presented in \Cref{cor:population-error} (with $\vw^*$ interpreted as $\vw^*_{\hat{\vmu}}$, the minimizer for the empirical loss on the uncentered stable samples) are preserved in their order of magnitude.

\section{Proof of Corollary \ref{cor:population-error} in Section \ref{sec:main-algo}}\label{sec:population-loss}
We restate \Cref{cor:population-error} for completeness:
\populationmain*

To prove this corollary, we leverage the following uniform convergence result:
\begin{fact}[{\cite[Corollary 5]{kakade2008complexity}}]\label{fact:uniform-convergence}
  Fix $W > 0$ and $0 < \tau < 1$. Suppose we are given $N$ samples \(\{(\vx_i, y_i)\}\)  independently drawn from a distribution $\sP_0$,  such that for some constant $C$, $\norm{\vx}_2 \le C \sigma \sqrt{d}$ holds $\sP_0$-almost surely. Then, with probability at least $1-\tau$, for all $\vw$ satisfying $\norm{\vw}_2 \le W$, we have
  \( \E_{(\vx, y) \sim \sP_0} [\ell_{y}(\vx \cdot \vw)] \le \frac1N \sum_{i=1}^{N} \ell_{y_i}(\vx_i \cdot \vw) + C \zeta \sigma W \sqrt{(16d + d\log(1/\tau)) / N}\). 
\end{fact}

\begin{proof}[ of \Cref{cor:population-error}]
Let $\bar F(\vw) = \E_{(\vx,y) \sim \sP_0} [\ell(\vw; \vx, y)] + \rho \zeta \norm{\vw}_s$. By \eqref{eq:w1-dro-regularization}, this is the population $W_1$-DRO objective: $\max_{\sQ: W_1^c(\sP_0, \sQ) \le \rho} \E_{\vxi \sim \sQ}[\ell(\vw; \vxi)]$. Denote by ${\bar\vw}^* = \argmin_{\vw: \norm{\vw}_2 \le W} \bar F(\vw)$.

Let $\hat{F}_N(\vw) = \frac{1}{N} \sum_{i=1}^N \ell_{y_i}(\vx_i \cdot \vw) + \rho \zeta \norm{\vw}_s$ be the empirical objective based on $N$ clean samples $\{\vx_i, y_i\}_{i=1}^N$ drawn from $\sP_0$.
We aim to bound $\bar F(\vw_T) - \inf_{\vw': \norm{\vw'}_2 \le W} \bar F(\vw')$. 

Let $F'_S(\vw) = \frac1{\abs{S'}}\sum_{i: \tilde\vx_i \in S'} \ell_{y_i}(\vx_i \cdot \vw) + \rho \zeta \norm{\vw}_s$ be the empirical objective based on the set of stable covariates, denoted by $S'$, whose existence is guaranteed with probability at least $1-\tau/2$ by \Cref{fact:bounded_variance_stability_iid}. 
\Cref{thm:main-empirical} implies that \Cref{alg:main-opt}, when run with $\vw_0 = \vzero$ on $\epsilon$-corrupted data, outputs $\vw_T$ such that
\begin{equation}\label{eq:cor_proof_algo_bound_final}
    \hat{F}_N(\vw_T) \le \min_{\vw': \norm{\vw'}_2 \le W} \hat{F}_N(\vw') + O(\sigma \zeta \sqrt\epsilon  \norm{\vw^*}_2).
\end{equation}
Using \Cref{fact:uniform-convergence}, with probability at least $1 - \tau/2$, it holds that
\begin{equation}\label{eq:cor_proof_uc_final}
    \sup_{\vw: \norm{\vw}_2 \le W} |\bar F(\vw) - \hat{F}_N(\vw)| = O(\zeta W \sqrt{ d \log(1/\tau)/N}).
\end{equation}
Combining these results with a union bound, with probability at least $1-\tau$, we have
\begin{align*}
    \bar F(\vw_T) &\stackrel{\eqref{eq:cor_proof_uc_final}}{\le} \hat{F}_N(\vw_T) + O(\zeta W \sqrt{ d/N}) 
    \\
    &\stackrel{\eqref{eq:cor_proof_algo_bound_final}}{\le} \left( \min_{\vw': \norm{\vw'}_2 \le W} \hat{F}_N(\vw') \right) + O(\sigma \zeta \sqrt{\epsilon} \norm{\vw^*}_2) + O(\zeta W \sigma\sqrt{ d/N}) 
    \\
    &\le \hat{F}_N({\bar\vw}^*) +  O(\sigma \zeta \sqrt{\epsilon} \norm{\vw^*}_2) + O(\zeta W \sigma \sqrt{ d/N}) 
    \\
    &\stackrel{\eqref{eq:cor_proof_uc_final}}{\le} \bar F({\bar\vw}^*) + O( \sigma \zeta \sqrt{\epsilon} \norm{\vw^*}_2)+ O(\zeta W \sigma \sqrt{ d/N}) \;.
\end{align*}
Recall that $\bar F({\bar\vw}^*) = \inf_{\vw: \norm{\vw}_2 \le W} \bar F(\vw)$, the excess risk is bounded by $O(\sigma \zeta \sqrt{\epsilon}\norm{\vw^*}+ \zeta W \sqrt{ d/N})$.
Substituting $N = O(\zeta W^2 d \log(d)/\epsilon)$, the uniform convergence error term $O(\zeta W \sqrt{ d/N})$ becomes
$$O(\zeta W \sigma \sqrt{ d/N}) = O(\zeta \sigma\sqrt\epsilon / \log d) = O(\zeta \sigma\sqrt\epsilon),$$
which is dominated by the optimization error term $O( \sigma \zeta \sqrt{\epsilon} \norm{\vw^*}_2)$. 
\end{proof}

\newrefcontext[sorting=nyt]
\printbibliography
\end{document}